\documentclass[10pt,twocolumn,letterpaper]{article}

\usepackage{iccv}
\usepackage{times}
\usepackage{amsmath}
\usepackage{amssymb}
\usepackage{natbib}
\usepackage{multirow}

\usepackage[utf8]{inputenc} 
\usepackage[T1]{fontenc}    
\usepackage{url}            
\usepackage{booktabs}       
\usepackage{amsfonts}       
\usepackage{nicefrac}       
\usepackage{microtype}      
\usepackage{xcolor}         
\usepackage{dsfont}
\usepackage{amsthm}
\usepackage{thmtools}
\usepackage{subcaption}
\usepackage{graphicx}
\usepackage[ruled,vlined]{algorithm2e}
\usepackage{wrapfig}
\usepackage{enumitem}
\usepackage{colortbl}
\usepackage{bm}
\usepackage{textcomp,gensymb}
\usepackage{pifont}
\usepackage{threeparttable}
\usepackage{pgfplotstable}
\usepackage{thm-restate}
\usepackage[accsupp]{axessibility}  
\usepackage{blindtext}

\newcommand{\cmark}{\ding{51}}
\newcommand{\xmark}{\ding{55}}

\def \Hf {\Gamma}

\def \zero {\mathbf{0}}
\def \rv {{\bm r}}
\def \xv {{\bm x}}
\def \yv {{\bm y}}

\def \thetav {{\bm \theta}}
\def \Hv {{\bm H}}
\def \Dv {{\bm D}}

\newcommand{\argmin}{\mathop{\mathrm{argmin}}}

\newtheorem{thm}{Theorem}
\newtheorem{definition}[thm]{Definition}

\usepackage[breaklinks=true,bookmarks=false,hidelinks]{hyperref}
\usepackage{cleveref}

\iccvfinalcopy 

\def\iccvPaperID{11185}

\ificcvfinal\fi
\pgfplotsset{compat=1.18}
\begin{document}

\title{Understanding Hessian Alignment for Domain Generalization}

\author{Sobhan Hemati$^*$ \qquad Guojun Zhang$^*$ \qquad Amir Estiri  \qquad Xi Chen\\
Huawei Noah’s Ark Lab \\
{\small \texttt{\{sobhan.hemati, guojun.zhang, amir.hossein.estiri1, xi.chen4\}@huawei.com}}
}

\maketitle
\def\thefootnote{*}\footnotetext{Equal contribution. ICCV 2023 camera ready. }
\ificcvfinal\thispagestyle{empty}\fi

\begin{abstract}

Out-of-distribution (OOD) generalization is a critical ability for deep learning models in many real-world scenarios including healthcare and autonomous vehicles. Recently, different techniques have been proposed to improve OOD generalization. Among these methods, gradient-based regularizers have shown promising performance compared with other competitors. Despite this success, our understanding of the role of Hessian and gradient alignment in domain generalization is still limited. To address this shortcoming, we analyze the role of the classifier's head Hessian matrix and gradient in domain generalization using recent OOD theory of transferability. Theoretically, we show that spectral norm between the classifier's head Hessian matrices across domains is an upper bound of the transfer measure, a notion of distance between target and source domains. Furthermore, we analyze all the attributes that get aligned when we encourage similarity between Hessians and gradients. Our analysis explains  the success of many regularizers like CORAL, IRM, V-REx, Fish, IGA, and Fishr as they regularize part of the classifier's head Hessian and/or gradient. Finally, we propose two simple yet effective methods to match the classifier's head Hessians and gradients in an efficient way, based on the Hessian Gradient Product (HGP) and Hutchinson's method (Hutchinson), and without directly calculating Hessians. We validate the OOD generalization ability of proposed methods in different scenarios, including transferability, severe correlation shift, label shift and diversity shift. Our results show that Hessian alignment methods achieve promising performance on various OOD benchmarks. The code is available \href{https://github.com/huawei-noah/Federated-Learning/tree/main/HessianAlignment}{\color{blue}here}.

\end{abstract}

\begin{figure*}[h]
\centering
\centering\includegraphics[width=0.99\textwidth]{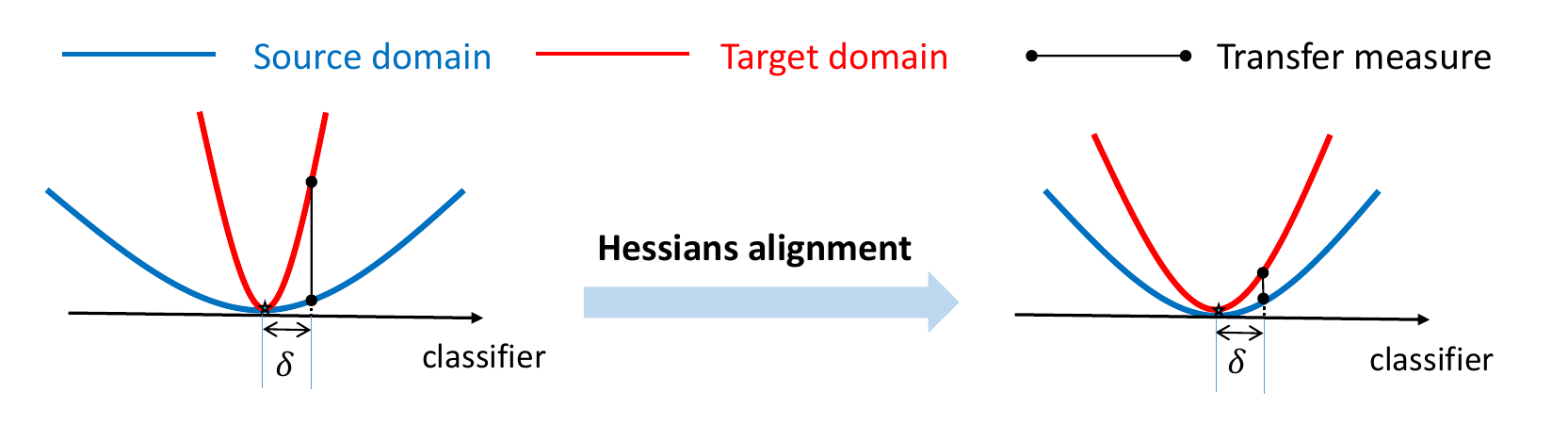}
\caption{Visualization of Hessian alignment for domain generalization. $\delta$ denotes the small deviation from an optimum. Hessian alignment matches the curvature and thus improves the transfer measure between source and target domains. }
\label{fig:intuition}
\end{figure*}

\section{Introduction}

A typical assumption in the design of current supervised deep learning algorithms is identical distributions of test and training data. Unfortunately, in real-world problems, this assumption is not always true \citep{koyama2020out} and the distribution gap between test and training data degrades the performance of deep learning models \citep{wang2022generalizing}. In practice, not only is there a distribution shift between test and training data, but also the training data does not follow the i.i.d.~assumption but contains data from multiple domains. Histopathology medical datasets are a concrete example where each hospital is using different types of staining procedures and/or scanners \citep{chang2021stain}. In such a scenario, a deep learning model often fails to generalize its knowledge to unseen domains. \citet{arjovsky2019invariant} argue that the main reason behind poor OOD generalization
is the tendency of deep learning models to capture simple spurious correlations instead of real causal information. In order to improve OOD generalization, the learning algorithm has to learn from invariant mechanisms.

The problem setup in domain generalization (DG) assumes the training set contains data from multiple sources (domains) where the task remains the same across different data sources \citep{muandet2013domain}. Furthermore, we assume there is a causal mechanism that remains invariant across different domains and hopefully for OOD data.  Although many learning algorithms have been proposed to capture invariance, a recent work by \citet{gulrajani2020search}~shows that given a standard experimental setting, the classic Empirical Risk Minimization (ERM) algorithm \citep{vapnik1999overview}, which minimizes the average of losses across different training domains, outperforms many recently proposed domain generalization methods. This suggests that many current algorithms may not be successful in capturing the invariance in data.

Out of the effective DG algorithms, one recent and promising line of research is gradient-based methods which try to capture invariance in gradient space \citep{parascandolo2020learning,koyama2020out,shi2021gradient,rame2022fishr}. One fast and simple explanation for this success can be derived from the seminal work by \citet{jaakkola1998exploiting}. According to this work, given a generative model, the gradient with respect to each parameter quantifies the role of that parameter in the generation of a data point. With this intuition, encouraging the similarity between gradients across environments can be translated as enforcing the network parameters to contribute the same for different environments which captures a notion of invariance. Beyond gradient matching, the Invariant Learning Consistency (ILC) measure \citep{parascandolo2020learning} motivates matching Hessians across different environments after convergence, given that we found the optimal solution for all environments and that the Hessian matrix is diagonal. Although the efforts in \citet{parascandolo2020learning,koyama2020out,rame2022fishr} improve our understanding of the role of Hessians and gradients to some extent, the ILC measure is built based on heuristics and is only valid under restricted assumptions like diagonal Hessian matrices. Moreover, it is not exactly clear what attributes are getting aligned when we match gradients or even Hessians. Finally, the proposed methods to align gradients or Hessians (e.g., ANDmask, \citealt{parascandolo2020learning}) seem to be suboptimal as discussed by follow-up works \citep{shahtalebi2021sand, rame2022fishr}.

To address these limitations, in this paper, we study the role of Hessians and gradients in domain generalization. Since Hessians for the whole neural networks are hard to compute, we focus on Hessians of the classifier head, which contains more information than one usually expects. 
We summarize our contributions as follows:
\begin{itemize}
\item To justify Hessian matching, we utilize a recent concept from statistical learning theory, called transfer measure \citep{zhang2021quantifying}, which describes transferability between target and source domains. Unlike ILC, transfer measure avoids the restrictive assumptions of diagonal Hessians. We theoretically show that the distance between the classifier's head Hessians is an upper bound of the transfer measure. 

\item Furthermore, we show that  Hessians and gradient alignment can be treated as feature matching. Our analysis of feature matching compares other DG algorithms like CORAL \citep{sun2016deep} and V-REx \citep{krueger2021out} in a unified framework. Especially, the success of CORAL can be attributed to approximate Hessian alignment using our analysis.

 \item Last but not least, to match Hessians efficiently, we propose two simple yet effective methods, based on different estimation of the Hessian. To our knowledge, these methods are the first DG algorithms based on Hessian estimation that align Hessians and gradients simultaneously.
\end{itemize}

Figure~\ref{fig:intuition} provides an intuitive explanation why Hessian alignment can reduce domain shift through minimizing transferability. With Hessian matching, the transfer measure, or the excess risk gap between source and target domains, is minimized. This makes any near-optimal source classifier to be also near-optimal on the target, and thus improves the transferability and OOD generalization.

\section{Problem Setting and Related Work}

Consider a deep neural network that consists of a feature extractor $g$ and a classifier head $h_{\thetav}$ with parameters $\thetav$, and suppose the training data contains $n$ different source domains $\mathcal{E}^{source} = \{\Sc_1, \dots, \Sc_n\}$. The main goal in DG is to design a learning algorithm where the trained model can generalize to an unseen target domain $\Tc$. Denoting $\Hc$ as the hypothesis class, the classification loss of a classifier $h_{\thetav} \in \Hc$ on a domain $\Dc$ is:
\begin{align}\label{eq:1}
\mathcal{L}_{\Dc}(\thetav) = \Eb_{(\xv, \yv)\sim \Dc} [\ell(\hat{\yv}, \yv;\thetav)],
\end{align}
where $\xv$ and $\yv$ are the input and the associated one-hot label, $\hat{\yv}$ is the logit (classifier output) and $\hat{\yv} = h_{\thetav}(g(\xv))$ where $g(\xv)$ is the output of a feature extractor, $h_{\thetav}$ is the classifier and $\ell(\hat{\yv}, \yv;\thetav)$ is the cross entropy between $\hat{\yv}$ and $\yv$. The classic baseline for OOD generalization is Empirical Risk Minimization \citep[ERM,][]{vapnik1999overview} which minimizes the average of losses across different training environments namely $\frac{1}{n} \sum_e \mathcal{L}_{\Sc_e}$. However, it has been shown that this approach might fail to generalize to out-of-distribution data when there is an easy-to-learn spurious correlation in data \citep{arjovsky2019invariant}. To tackle this, many algorithms are proposed to avoid domain-specific representations and capture invariance of data from different perspectives. We categorize the most related into four main classes:

\textbf{Data augmentation.}
Recently many papers have shown that data augmentation can improve OOD generalization. \citet{gulrajani2020search} showed that the classic ERM with data augmentation can outperform many DG algorithms. \citet{zhang2018mixup} proposed \emph{mixup}, a data augmentation strategy designed for domain generalization where  samples across multiple source domains get mixed to generate new data for training. \citet{ilse2021selecting} derived an algorithm that select proper data augmentation.

\textbf{Robust optimization} techniques have been used to achieve better OOD generalization \citep{sagawa2019distributionally,hu2018does}. These techniques minimize the worst-case loss over a set of source domains, and can be helpful when we know the distance between training and test environments.

\textbf{Learning domain invariant features.} Invariant representation learning was initialized by seminal work of \citet{ben2010theory}. A wide range of notions of invariance have been proposed since then. \citet{tzeng2014deep} proposed feature matching across domains while \citet{sun2016deep} considered encouraging the similarity between feature covariance matrices. DANN \citep{ganin2016domain} proposed matching feature distributions with an adversarial network, which is followed by MMD \citep{li2018domain} and MDAN \citep{zhao2018adversarial}.  \citet{zhao2019learning} discussed the fundamental tradeoff between feature matching and optimal joint risk.
Another line work started from Invariant Risk Minimization (IRM) \citep{arjovsky2019invariant}, which proposed to encourage the classifier head to be optimal for different domains simultaneously. This idea led to a new notion of invariance as followed by e.g., Risk Extrapolation (V-REx) \citep{krueger2021out} where authors proposed similarity between losses across training domains. 
 
\textbf{Gradient matching based algorithms.} Recently, gradient alignment has been used for domain generalization \citep{parascandolo2020learning,shahtalebi2021sand,koyama2020out,Mansilla_2021_ICCV,shi2021gradient,rame2022fishr} with good performance. In this regard, \citet{shi2021gradient} proposed Fish to maximize the inner product of average gradients for different domains. In a similar approach, \citet{koyama2020out} proposed the Inter Gradient Alignment (IGA) algorithm that minimizes the variance of average gradients over environments. \citet{parascandolo2020learning} developed ANDmask algorithm where the authors proposed ILC, which measures the consistency of local minima for different environments. \citet{parascandolo2020learning} proposed ANDmask to reduce the speed of convergence in directions where landscapes have different curvatures. Inspired by ILC measure \citep{parascandolo2020learning}, 
\citet{rame2022fishr} proposed Fishr as a better alternative compared with ANDmask to reduce inconsistency. In Fishr, the variance of domain level gradients is minimized.

\section{Theory of Hessian Matching}
\noindent In this section, we study the role of Hessians and gradients in domain generalization from two views: transferability and invariant representation learning.

\subsection{Hessian alignment minimizes transfer measure}\label{sec:alignment}
\noindent We motivate our work through the lens of transferability and transfer measure introduced in \citet{zhang2021quantifying}. By their definition, domain $\Sc$ is transferable to domain $\Tc$ on a hypothesis class $\Hc$, if given any near-optimal classifier $h_{\thetav}$ on $\Sc$, it also achieves near-optimal performance on $\Tc$. More precisely, we have:

\begin{definition}[\textbf{transferability, \citealt{zhang2021quantifying}}]\label{def:transfer}
 
$\Sc$ is $(\delta_\Sc, \delta_\Tc)_{\Hc}$-transferable to $\Tc$ if for $\delta_\Sc > 0$, there exists $\delta_\Tc > 0$ such that ${\argmin} (\mathcal{L}_\Sc, \delta_\Sc)_\Hc \subseteq {\argmin} (\mathcal{L}_\Tc, \delta_\Tc)_\Hc$, where:
\begin{align}
{\argmin} (\mathcal{L}_\Dc, \delta_\Dc)_\Hc &:= \{h_{\thetav}\in \Hc: \mathcal{L}_\Dc(\thetav) \nonumber \\
&\leq \inf_{h_{\thetav}\in \Hc}\mathcal{L}_\Dc(\thetav) + \delta_\Dc\}, \Dc \in \{\Sc, \Tc\}. \nonumber
\end{align}
\end{definition}

\noindent The set ${\argmin} (\Lc_\Dc, \delta_\Dc)_\Hc$ is called  a \emph{$\delta$-minimal set} \citep{koltchinskii2010rademacher} of $\mathcal{L}_\Dc$, and represents the near-optimal set of classifiers. With the $\delta$-minimal set, we can define the transfer measure:

\begin{definition}[\textbf{transfer measure, \citealt{zhang2021quantifying}}]\label{def:transfer_measure}
Given $\Hf = \argmin(\Lc_\Sc, \delta_\Sc)_{\Hc}$, $\Lc_\Sc^* := \inf_{h_{\thetav}\in \Hf} \mathcal{L}_\Sc(\thetav)$ and $\mathcal{L}_\Tc^* := \inf_{h_{\thetav}\in \Hf} \mathcal{L}_\Tc(\thetav)$ we define the transfer measure $\mathtt{T}_{\Hf}(\Sc\|\Tc)$ as:
\begin{align}\label{eq:2}
&\mathtt{T}_{\Hf}(\Sc\|\Tc) := \sup_{h_{\thetav}\in \Hf} \mathcal{L}_\Tc(\thetav) - \mathcal{L}_\Tc^* - (\mathcal{L}_\Sc(\thetav) - \mathcal{L}_\Sc^*).
\end{align}
\end{definition}
\noindent The transfer measure in Eq.~\ref{eq:2} measures the upper bound of the difference between source domain excess risk ($\mathcal{L}_\Sc(\thetav) - \mathcal{L}_\Sc^*$) and the target domain excess risk ($\mathcal{L}_\Tc(\thetav) - \mathcal{L}_\Tc^*$), on a neighborhood $\Gamma$. Usually, we choose $\Gamma$ to be a neighborhood of our learned near-optimal source classifier.

Now, we state our main theorem which, under mild assumptions, implies that Hessian alignment effectively reduces the transfer measure and improving transferability.

\begin{thm}[\textbf{Hessian alignment}]\label{thm:Hessian_distance_spectral_norm}
Suppose both the source and target domain losses $\Lc_\Sc$ and $\Lc_\Tc$ are $\mu$-strongly convex with respect to the classifier head and twice differentiable with the same minimizer, i.e., $\argmin \Lc_\Sc = \argmin \Lc_\Tc$. Then with $\delta = 2\delta_\Sc/\mu$ and $\Hf = \argmin(\Lc_\Sc, \delta_\Sc)_{\Hc}$ we have:
\begin{align}
\mathtt{T}_{\Hf}(\Sc\|\Tc) \leq \ \frac{1}{2}\delta^2 \|\Hv_\Tc  - \Hv_\Sc\|_2 + o(\delta^2).
\end{align}
\end{thm}

\begin{proof}
Let ${\thetav}^*$ represent the optimal classifier for both source and target domains. Given that $ \mathcal{L}_\Sc(\thetav)$ is strongly convex, we can write $\inf_{h_{\thetav}\in \Hc}\mathcal{L}_\Sc(\thetav)= \mathcal{L}_\Sc(\thetav^*)$ and subsequently $\Hf$ can be written as
\begin{align}
{\argmin} (\mathcal{L}_\Sc, \delta_\Sc)_\Hc := \{h_{\thetav}\in \Hc: \mathcal{L}_\Sc({\thetav}) - \mathcal{L}_\Sc(\thetav^*) \leq \delta_\Sc\}. \nonumber 
\end{align}
Now, we define set $\Fc_1$ as
\begin{align}
\label{eq:3}
\Fc_1= \{\thetav: \mathcal{L}_\Sc(\thetav) - \mathcal{L}_\Sc(\thetav^*) \leq \delta_\Sc\}. 
\end{align}

On the other hand, from the $\mu$-strong convexity of $\Lc_\Sc(\thetav)$, we can write
\begin{align}\label{eq:4}
\mathcal{L}_\Sc(\thetav) \geq \Lc_\Sc(\thetav^{*}) +  \nabla_{\thetav} \mathcal{L}_\Sc^{\top}(\thetav^{*}) (\thetav- \thetav^{*}) + \frac{\mu}{2} \|\thetav-\thetav^*\|_{2}^{2}.
\end{align}
Given the optimality of $\thetav^{*}$ (and thus $\nabla_{\thetav} \mathcal{L}_\Sc(\thetav^{*})=\zero$), from Eqs.~\ref{eq:3} and \ref{eq:4} we obtain that for any $\thetav \in \Fc_1$, 
\begin{align}\label{eq:5}
&  \frac{\mu}{2} \|\thetav-\thetav^*\|_{2}^{2} \leq \mathcal{L}_\Sc(\thetav) - \Lc_\Sc(\thetav^{*}) \leq \delta_\Sc.
\end{align}
If we define set $\Fc_2$ as
\begin{align}
\label{eq:6}
\Fc_2= \{\thetav: \frac{\mu}{2} \|\thetav-\thetav^*\|_{2}^{2}  \leq \delta_\Sc\}. 
\end{align}
Then \eqref{eq:5} implies $\Fc_1 \subset \Fc_2$ and as a result, an upper bound of the transfer measure $\mathtt{T}_{\Hf}(\Sc\|\Tc)$ can be written as
\begin{align}\label{eq:7}
&\sup_{\|\thetav - \thetav^*\|_2 \leq \delta}\mathcal{L}_\Tc(\thetav) - \Lc_\Tc(\thetav^{*}) - (\mathcal{L}_\Sc(\thetav) - \Lc_\Sc(\thetav^{*})),
\end{align}
with $\delta = 2\delta_\Sc/\mu$. Now, we may write the second-order Taylor expansion $\mathcal{L}_\Sc$ around $\thetav^{*}$ as:
\begin{align}
\begin{split}
\label{eq:8}
\mathcal{L}_\Sc(\thetav) & =  \Lc_\Sc(\thetav^{*}) + (\thetav-\thetav^{*})^{\top} \nabla_{\thetav} \mathcal{L}_\Sc(\thetav^{*}) + \\
&+ \frac{1}{2} (\thetav-\thetav^{*})^{\top} \Hv_\Sc (\thetav-\thetav^{*}) + o(\|\thetav - \thetav^*\|^2),
\end{split}
\end{align}
where given $\nabla_{\thetav} \mathcal{L}_\Sc(\thetav^{*})= \zero$, the source domain excess risk is 
\begin{align}\label{eq:9}
& \mathcal{L}_\Sc(\thetav) - \Lc_\Sc(\thetav^{*}) =  \frac{1}{2} \Delta\thetav^{\top} \Hv_\Sc \Delta\thetav + o(\|\thetav - \thetav^*\|^2),
\end{align}
with $\Delta \thetav = \thetav - \thetav^*$.
Following a similar path for the target domain and considering that $\thetav^*$ is the optimal solution for both source and target domains, the transfer measure $\mathtt{T}_{\Hf}(\Sc\| \Tc)$ is upper bounded by
\begin{align}\label{eq:10}
\sup_{\|\thetav - \thetav^*\|_2 \leq \delta} \frac{1}{2} (\thetav-\thetav^{*})^{\top}  (\Hv_\Tc  - \Hv_\Sc) (\thetav-\thetav^{*}) +  o(\delta^2),
\end{align}
which, from the definition of spectral norm, results in 
\begin{align}\label{eq:11}
&\mathtt{T}_{\Hf}(\Sc\|\Tc) \leq \ \frac{1}{2}\delta^2 \|\Hv_\Tc  - \Hv_\Sc\|_2 + o(\delta^2),
\end{align}
and the proof is complete.
\end{proof}

Note that convexity and differentiability are easily satisfied if we use the standard linear classifier head and cross entropy loss with softmax. Since in practice we do not have access to the target domain, we minimize the distance between available source domains. For simplicity, we replace the spectral norm with the Frobenius norm (note that Frobenius norm is an upper bound, i.e., $\| \cdot \|_2 \leq \|\cdot \|_F $). As opposed to pairwise regularizers which grow quadratically with the number of environments, we minimize the variance of Hessian distances across domains, namely $$\frac{1}{n} \sum_{e=1}^n \| \Hv_{\Sc_e}  - \overline{\Hv_{\Sc}} \|_2^2\mbox{ where } \overline{\Hv_{\Sc}}=\frac{1}{n} \sum_{e=1}^n \Hv_{\Sc_e}$$ and $\Hv_{\Sc_e}$ is the Hessian matrix for the $e$-th source domain. In this way, computation is linear with respect to the number of environments.

\begin{table*}[]
 \caption{ 
    The alignment attributes for different Domain Generalization algorithms where  Loss is $(y_p - \hat{y}_p)^2$, Feature $z_q$,  Covariance $z_q z_v$,  Error \ $y_p - \hat{y}_p$, Error $\times$ Feature \ $(\hat{y}_p - y_p) z_q$,  Logit \  $\hat{y}_u (\delta_{u,v}-\hat{y}_v)$, Logit $\times$ Feature \ $z_q \hat{y}_p (\delta_{p,u}-\hat{y}_u)$,  and Logit $\times$  Covariance is $\hat{y}_p z_q z_v (\delta_{p,u}-\hat{y}_u)$. Gradient alignment method matches Error and Error $\times$ Feature while matching Hessians would align Logit, Logit $\times$ Feature, Logit $\times$ Covariance. $^1$To be precise, the Loss that get aligned in V-Rex is negative log likelihood instead of squared error.
    $^2$To be precise, Fishr aligns Loss $\times$ Feature.} 
      \label{table:Table1}
        \centering
     \scalebox{0.9}{
        \begin{tabular}{l c c c c c c c c}
            \toprule Alignment attribute    &  Loss  &Feature & Covariance &Error & Error $\times$ Feature  & Logit &Logit $\times$ Feature      & Logit $\times$ Covariance                \\
            \midrule
            V-Rex
             &\cmark$^1$ &\xmark   &\xmark &\xmark
             &\xmark
             &\xmark  
             &\xmark
            &\xmark\\ 
             \midrule
            CORAL  &\xmark &\cmark   &\cmark &\xmark
             &\xmark
             &\xmark  
             &\xmark
            &\xmark\\
             \midrule

             IGA
            &\xmark &\xmark   &\xmark    &\cmark
             &\cmark
             &\xmark  
             &\xmark
            &\xmark\\
             \midrule
            Fish
            &\xmark &\xmark   &\xmark    &\cmark
             &\cmark
             &\xmark  
             &\xmark
            &\xmark\\
             \midrule
            Fishr
            &\cmark &\xmark   &\xmark    &\xmark
             &\cmark$^2$
             &\xmark  
             &\xmark
            &\xmark\\
                      \midrule
            Hessian Alignment
             &\xmark &\xmark   &\xmark   &\cmark
             &\cmark
             &\cmark  
             &\cmark
            &\cmark\\
            \bottomrule
        \end{tabular}}
\end{table*}

\subsection{Hessian and gradient alignment is feature matching}

\noindent In \S \ref{sec:alignment} we assumed that the minimizers of source and target are the same. In order to align them, gradient alignment is often necessary.
To get more insight into the attributes that get matched while the Hessians and gradients get aligned, in Proposition \ref{prop:1}, we study the Hessian and gradient structures with respect to the classifier head parameters.

\begin{restatable}[\textbf{Alignment attributes in Hessian and gradient}]{prop}{Alignment}\label{prop:1} 
Let $\hat{y}_{p}$ and $y_{p}$ be the network prediction and true target with the $p$-th class, $z_i$ be the $i$-th feature before the classifier. Denote the classifier's parameters as $w_{k,i}$, the element in row $k$ and column $i$ of the classifier weight matrix, and $b_k$ as the bias term for the $k$-th neuron. Matching the gradients and Hessians w.r.t.~the classifier head across domains aligns the following:
\begin{align}
\frac{\partial \ell}{\partial b_p} &= (\hat{y}_p - y_p), \ (Error) \label{eq:12} \\
\frac{\partial \mathcal{L}}{\partial w_{p,q}} &=  (\hat{y}_p - y_p) z_q, \ (Error \times Feature) \label{eq:13} \\
\frac{\partial ^2 \ell}{\partial b_u \partial b_v} &= \hat{y}_u (\delta_{u,v}-\hat{y}_v), \ (Logit) \label{eq:14}\\
\frac{\partial ^2 \ell}{\partial w_{p,q} \partial b_u} &= z_q \hat{y}_p (\delta_{p,u}-\hat{y}_u),  \ (Logit \times Feature) \label{eq:15}\\
\frac{\partial ^2 \ell}{\partial w_{p,q} \partial w_{u,v}} &= \hat{y}_p z_q z_v (\delta_{p,u}-\hat{y}_u), \ (Logit \times  Covariance) \label{eq:16}
\end{align}
where $\delta_{p,u}$ is the Kronecker delta which is 1 if $p=u$ and 0 otherwise and $u, v, p, q, k, i$ are dummy indices.
\end{restatable}
 
Here the error means the difference between the true one-hot label and our prediction. Eqs.~\ref{eq:12}--\ref{eq:16} show that matching Hessians and gradients with respect to classifier parameters will align the 
errors, features weighted by errors, logits, features weighted by logits, and covariance weighted by logits across different domains simultaneously. Although gradient alignment can be helpful before convergence, when the training is close to convergence, the gradients go to zero and aligning them is not helpful anymore. On the other hand, the attributes extracted from the Hessian matrix can remain non-zero both before and after convergence and aligning them can boost OOD performance.

Proposition \ref{prop:1} shows that matching gradients and Hessians can be seen as a generalization of other works like V-Rex \citep{krueger2021out}, CORAL \citep{sun2016deep}, Fish \citep{shi2021gradient}, IGA \citep{koyama2020out}, and Fishr \citep{rame2022fishr}, as these algorithms only partially realize the matching of loss/error, feature or covariance. We summarize alignments in different methods in Table~\ref{table:Table1} and reveal their connection to Hessian and gradient alignment. One interesting observation is that similar to Hessian alignment, CORAL also matches the feature and its covariance, which opens the venue to understand the success of CORAL with Hessian matching.
In the supplementary, we also present similar results for regression tasks.

\section{Efficient Hessian matching}

So far, we have shown that aligning Hessians across domains can reduce the transfer measure (or in other words, increase transferability) and match the representations at different levels i.e., logits, features, errors and covariances. However, given the computational complexity of calculating Hessian matrices, it is quite challenging to directly minimize $\| \Hv_{\Sc_1}  - \Hv_{\Sc_2} \|_F$. To align Hessian matrices efficiently, we propose two approaches, aligning Hessian-gradient products and matching Hessian diagonals using Hutchinson’s trace estimator \citep{bekas2007estimator}.

\subsection{Hessian Gradient Product (HGP)}
For the Hessian-gradient product matching we propose the following loss:
\begin{align} 
\label{eq:17} 
\mathcal{L}_{\rm HGP} & \!=\! \frac{1}{n} \sum_{e=1}^n \mathcal{L}_{\Sc_e} + \alpha  \|\Hv_{\Sc_e}\nabla_{\thetav}\mathcal{L}_{\Sc_e} - \overline{\Hv_{\Sc}\nabla_{\thetav}\mathcal{L}_{\Sc}}\|_{2}^{2} +  \nonumber \\
& +
\beta \|\nabla_{\thetav}\mathcal{L}_{\Sc_e} - \overline{\nabla_{\thetav}\mathcal{L}_{\Sc}}\|_{2}^{2}, 
\end{align}
where $\alpha$ and $\beta$ are regularization parameters, 
\begin{align}
\overline{\nabla_{\thetav}\mathcal{L}_{\Sc}} &= \frac1n\sum_{e=1}^n \nabla_{\thetav}\mathcal{L}_{\Sc_e}, \mbox{ and}\\
\overline{\Hv_{\Sc}\nabla_{\thetav}\mathcal{L}_{\Sc}} & = \frac1n\sum_{e=1}^n \Hv_{\Sc_e}\nabla_{\thetav}\mathcal{L}_{\Sc_e}.
\end{align}
Given that $\nabla_\thetav \mathcal{L}_{\Sc_e}$ and $\overline{\nabla_{\thetav}\mathcal{L}_{\Sc}}$ are close enough, minimizing eq.~\ref{eq:17} reduces the distance between $\Hv_{\Sc_e}$ and $\overline{\Hv_{\Sc}}$. We can efficiently compute the HGP term without directly calculating Hessians, 
using the following expression 
\begin{align}
\label{eq:18} 
\Hv_{\Sc_e}\nabla_{\thetav}\mathcal{L}_{\Sc_e} = \|\nabla_{\thetav}\mathcal{L}_{\Sc_e}\| \cdot \nabla_{\thetav} \|\nabla_{\thetav}\mathcal{L}_{\Sc_e}\|,
\end{align}
which follows from $\nabla_{\xv}\|\xv\|_2 = \xv/\|\xv\|_2$ with $\xv\in \Rb^d$ and chain rule. For domain $e$, to calculate the exact value of the Hessian-gradient product, we only need two rounds of backpropagation: one for $\nabla_{\thetav}\mathcal{L}_{\Sc_e}$ and another for $\nabla_{\thetav} \|\nabla_{\thetav}\mathcal{L}_{\Sc_e}\|$.

\subsection{Hutchinson’s method}
In order to estimate Hessians more accurately, we propose another approach to estimate the Hessian diagonal using Hutchinson’s method \citep{bekas2007estimator}:
\begin{align}
\label{eq:19} 
\Dv_{\Sc_e} = {\rm diag}(\Hv_{\Sc_e}) = \Eb_{\rv}[\rv \odot (\Hv_{\Sc_e} \cdot \rv)],
\end{align}
where $\Eb_{\rv}$ is expectation of random variable $\rv$  sampled from Rademacher distribution and $\odot$ is Hadamard product. In practice, the expectation is replaced with the mean over a set of samples. Note that $\Hv_{\Sc_e} \cdot \rv$ can be calculated efficiently by calculating  $\nabla_{\thetav} \left( \nabla_{\thetav}\mathcal{L}_{\Sc_e}^{\top} \rv\right)$.
Given that we have estimated the Hessian diagonal, the $\mathcal{L}_{\rm Hutchinson}$ loss is 
\begin{align}
\label{eq:20} 
\mathcal{L}_{\rm Hutchinson} = &\frac{1}{n} \sum_{e=1}^n \mathcal{L}_{\Sc_e} + \alpha \|\Dv_{\Sc_e} - \overline{\Dv_{\Sc}}\|_{2}^{2} + \nonumber \\
&\beta \|\nabla_{\thetav}\mathcal{L}_{\Sc_e} - \overline{\nabla_{\thetav}\mathcal{L}_{\Sc}}\|_{2}^{2}. 
\end{align}
Compared to HGP, Hutchinson's method is more expensive to compute, as it requires an estimation with sampling. In practice, we use 100 samples for each domain. The trade-off is that Hutchinson's method gives us a more accurate matching algorithm for Hessian, since after training, the Hessian matrix is often close to diagonal \citep{skorski2021modern}, whereas the gradient (and thus the Hessian-gradient product) becomes nearly zero.

\section{Experiments}

\noindent To comprehensively validate the effectiveness of HGP and Hutchinson, we evaluate the proposed algorithms from multiple views, in terms of transferability \citep{zhang2021quantifying}, OOD generalization on datasets with correlation shift (Colored MNIST), correlation shift with label shift (imbalanced Colored MNIST), and diversity shift \citep{ye2021ood}, including PACS \citep{li2017deeper} and OfficeHome \citep{venkateswara2017deep} datasets from DomainBed.

\subsection{Transferability}
\noindent To show the role of Hessian discrepancy in improving transferability (i.e, Theorem \ref{thm:Hessian_distance_spectral_norm}), we use the algorithm presented in \citet{zhang2021quantifying}, which computes the worst-case gap $\sup_{\|\thetav - \thetav^{*}\| \leq \delta} \mathcal{L}_{\Sc_i}(\thetav) - \mathcal{L}_{\Sc_j}(\thetav)$ among all pairs of $(i, j)$, to evaluate transferability among multiple domains. First, we find the domains $i$ and $j$ such that $\mathcal{L}_{\Sc_i}$ and $\mathcal{L}_{\Sc_j}$ are maximized and minimized respectively and consider the $\mathcal{L}_{\Sc_i} - \mathcal{L}_{\Sc_j}$ as the gap. Then, we run an ascent optimizer on classifier $\thetav$ to maximize gap and project $\thetav$ on to Euclidean ball $\|\thetav - \thetav^{*}\| \leq \delta$. We repeat this procedure for multiple rounds and report the test target accuracy for the worst-case gap. Figure~\ref{fig:attack} shows that OOD accuracies of both HGP and Hutchinson are robust against the attack of an ascent optimizer, which implies that minimizing $\| \Hv_{\Sc_1}  - \Hv_{\Sc_2} \|_F$ improves transferability.

\begin{figure}[ht]
  \centering
  \medskip
  \begin{subfigure}[t]{0.49\linewidth}
    \centering\includegraphics[width=4.2cm]{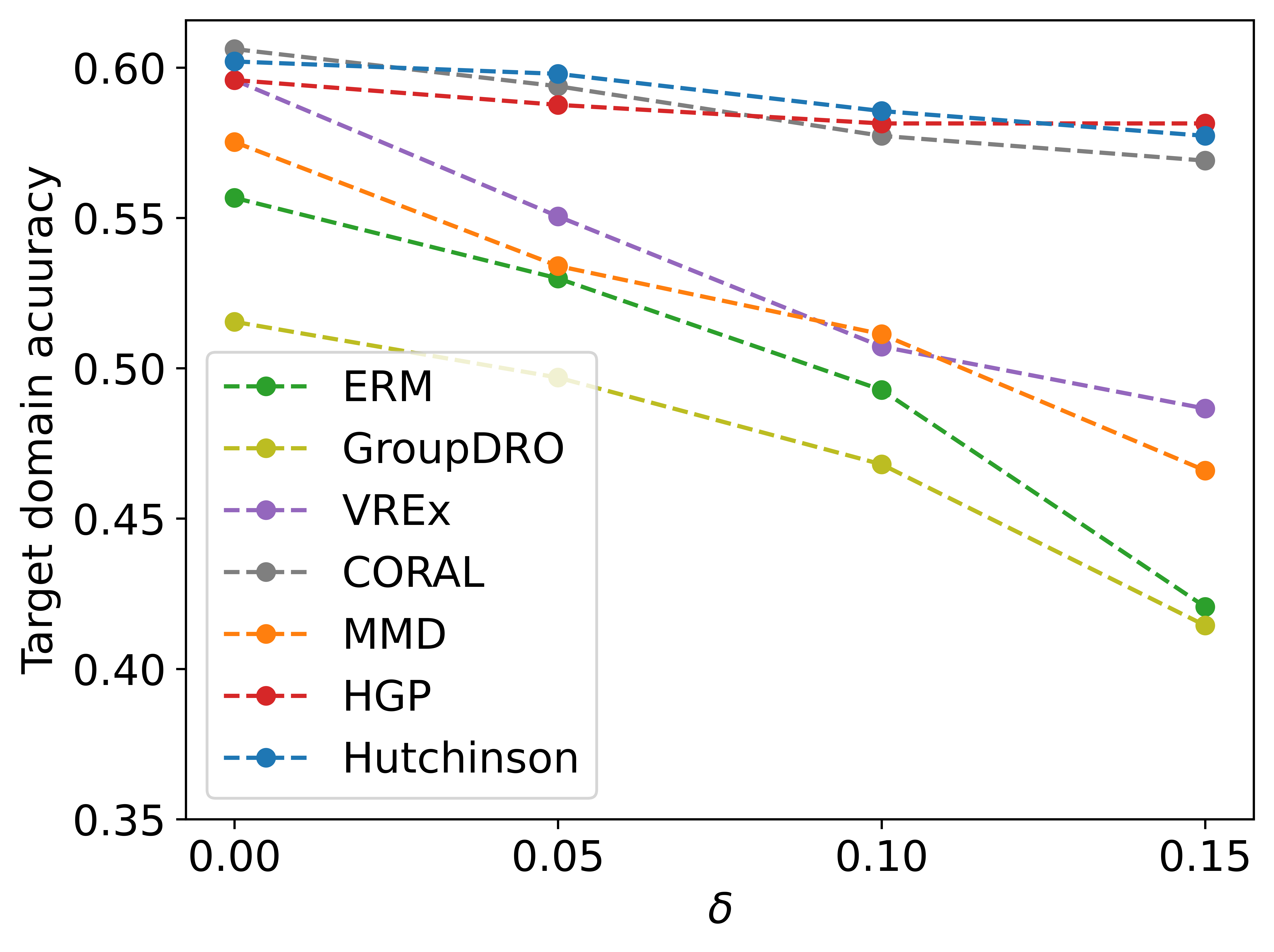}
    \caption{OfficeHome}
  \end{subfigure}
  \begin{subfigure}[t]{0.49\linewidth}
    \centering\includegraphics[width=4.2cm]{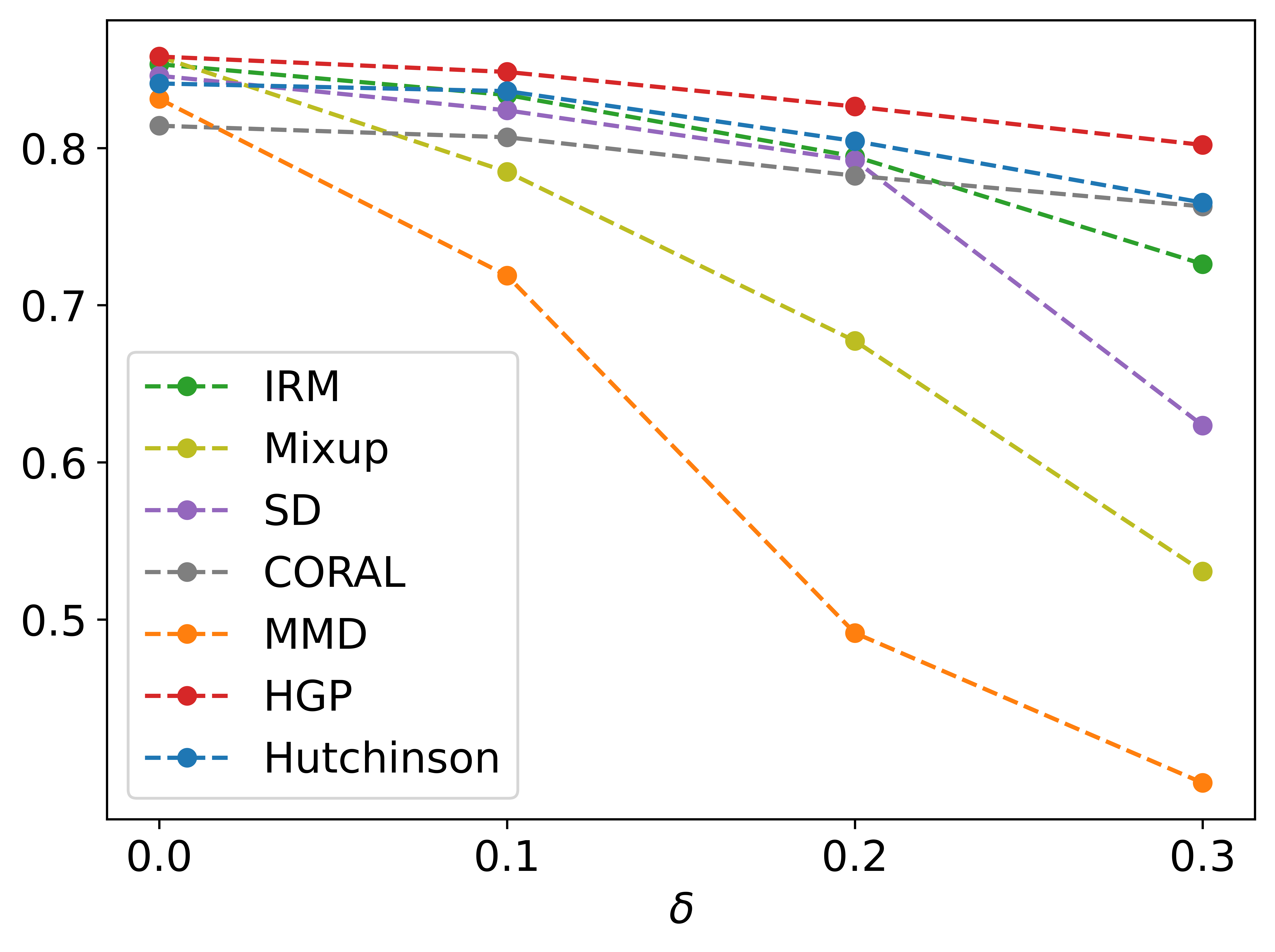}
    \caption{PACS}
    \end{subfigure}
\caption{OOD accuracy of HGP and Hutchinson compared with multiple baselines on OfficeHome and PACS datasets with different values of $\delta$ for the ascent optimizer.}
  \label{fig:attack}
\end{figure}

\subsection{OOD generalization under correlation shift using Colored MNIST}

\noindent Our Colored MNIST (CMNIST) setup is from  \citet{arjovsky2019invariant}, where the dataset contains two training and one test environments and the objective is a binary classification task in which class zero and one contain digits less than 5 and greater and equal 5 respectively. Further, the labels have been flipped with a probability 0.25 and each image has been colored either red or yellow such that the colors are heavily correlated with the class label. However, this color correlation is spurious as it has been reversed for test data. The correlations of labels with colors are +90\% and +80\% for training environments and -90\% for the test environment. 

Given the structure of data, if our learning algorithm only learns the spurious correlation (in this case color), the test accuracy will be 14\%, while if our model learns the invariant features, we can achieve 75\% test accuracy. The CMNIST setup \citep{arjovsky2019invariant} we use has also been followed by Fishr \citep{rame2022fishr} and V-Rex \citep{krueger2021out}. The only difference is that we replace the regularizer in IRM with our proposed regularizers in HGP and Hutchinson. In this setting, a simple fully connected network is used for training where we train the model for 501 steps and at step 190 the regularization parameters $\alpha$ and $\beta$ increase from 1 to 91257.18 suddenly (found by grid search for IRM). We repeated the experiments 10 times and reported the mean and average over these runs. 

\begin{table}[h]
\caption{Comparison of ERM, IRM, V-Rex, Fishr, and our proposed methods HGP and Hutchinson,  on Colored MNIST experiment in introduced in IRM \citep{arjovsky2019invariant} while the same hyperparameters have been used.}
 \label{table:Table2}
  \centering
  \begin{tabular}{lcc}
    \toprule
     Method & Train acc. & Test acc.    \\
    \toprule
        ERM & 86.4 ± 0.2 & 14.0 ± 0.7  \\ 
        IRM & 71.0 ± 0.5 & 65.6 ± 1.8  \\ 
        V-REx & 71.7 ± 1.5 & 67.2 ± 1.5  \\ 
        Fishr & 71.0 ± 0.9 & 69.5 ± 1.0 \\ 
        HGP & 71.0 ± 1.5 & 69.4 ± 1.3  \\
        Hutchinson & 61.7 ± 1.9 & \textbf{74.0 ± 1.2}  \\
    \bottomrule
  \end{tabular}
\end{table}

\begin{figure*}[ht]
  \centering
  \medskip
  \begin{subfigure}[t]{0.24\textwidth}
    \centering\includegraphics[width=4.22 cm]{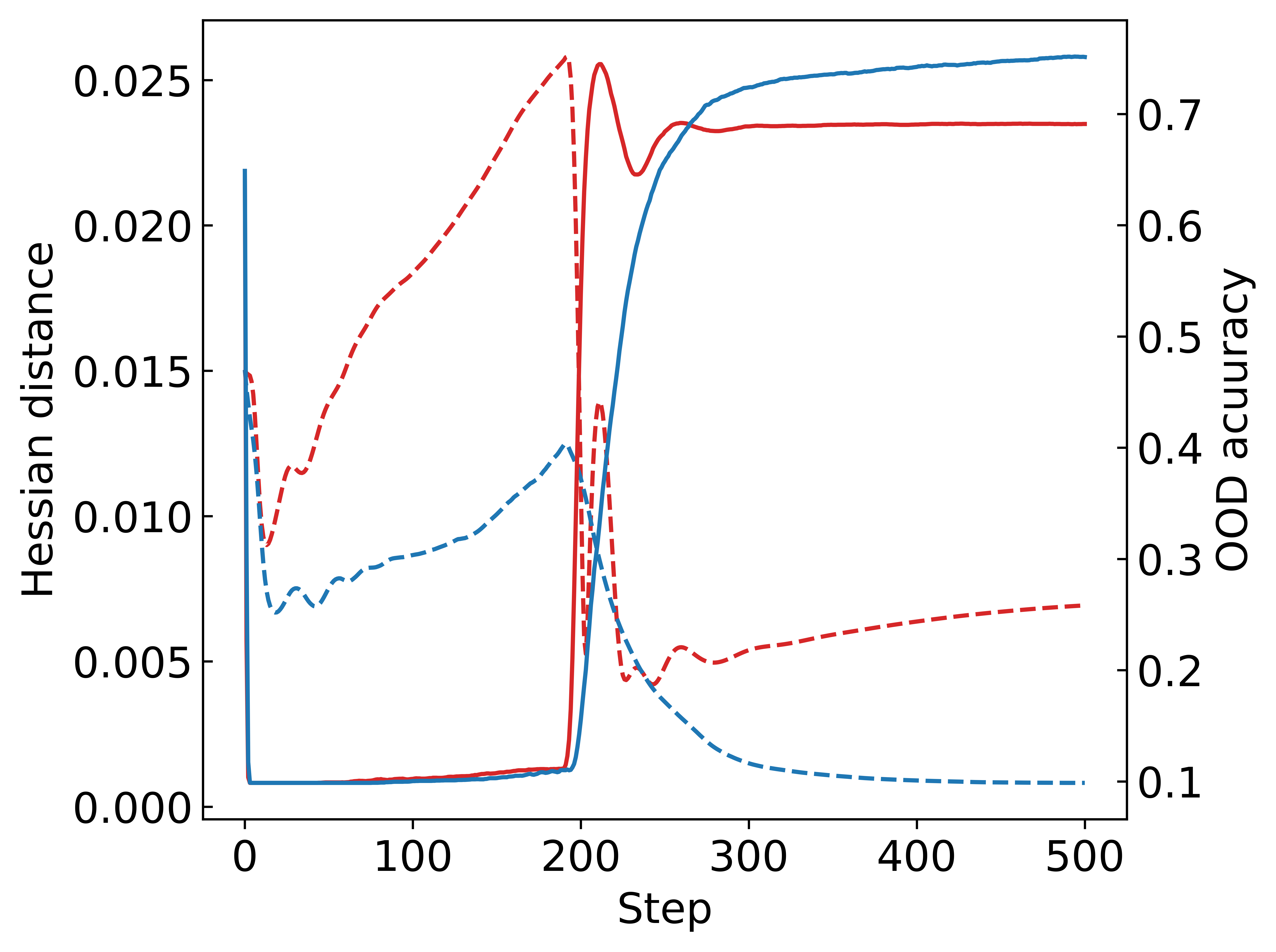}
    \caption{No label shift - accuracy}
    \label{figure:fig2-a}
  \end{subfigure}
\begin{subfigure}[t]{0.24\textwidth}
    \centering\includegraphics[width=4.22 cm]{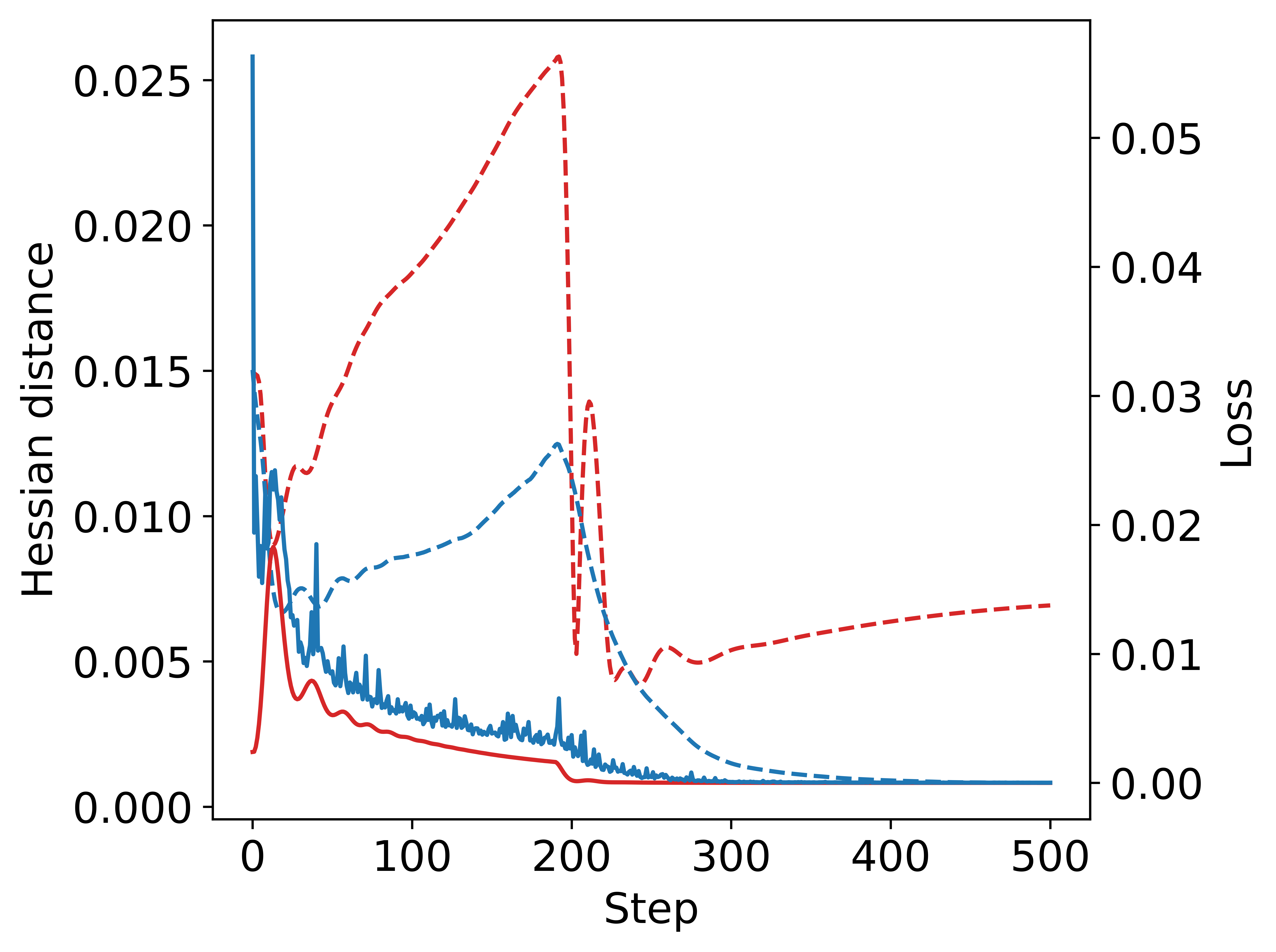}
    \caption{No label shift - loss}
    \label{figure:fig3-a}
    \end{subfigure}
\begin{subfigure}[t]{0.24\textwidth}
   \centering\includegraphics[width=4.22 cm]{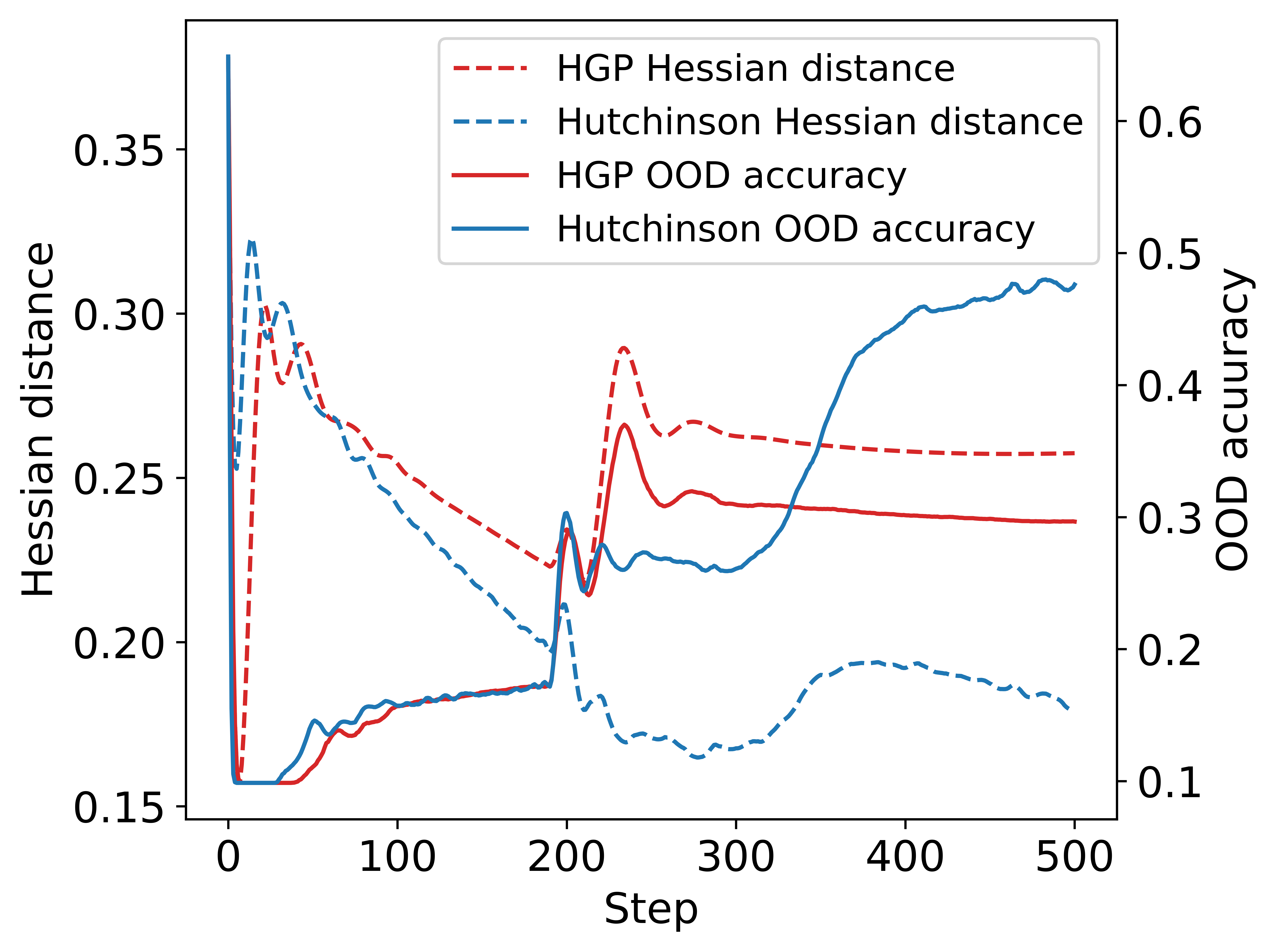}
    \caption{Heavy label shift - accuracy}
    \label{figure:fig2-b}
    \end{subfigure}
    \begin{subfigure}[t]{0.24\textwidth}
   \centering\includegraphics[width=4.22 cm]{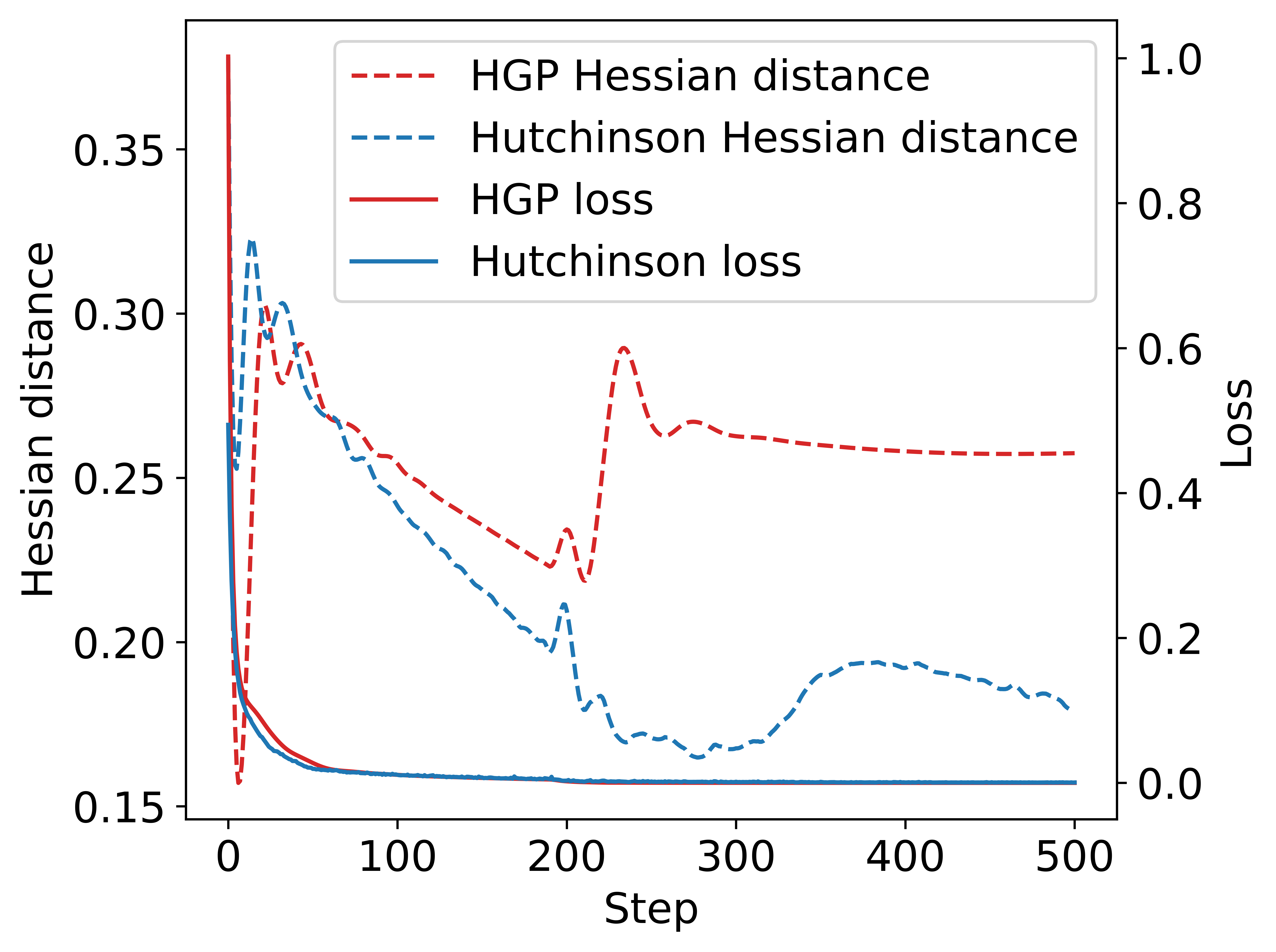}
    \caption{Heavy label shift - loss}
    \label{figure:fig3-b}
    \end{subfigure}
\caption{Correlation between Hessian distances ($\| \Hv_{\Sc_1}  - \Hv_{\Sc_2} \|_F$) and OOD accuracies/losses for HGP and Hutchinson regularization during the training for Colored MNIST. (a), (b): no label shift; (c), (d): heavy label shift. 
}
\label{fig:dynamics}
\end{figure*}
Table \ref{table:Table2} compares the performance of HGP and Hutchinson against common DG baselines. As can be seen, while HGP achieves competitive performance against the state-of-the-art Fishr algorithm, Hutchinson outperforms all methods by a large gap and obtains the near-optimal OOD accuracy of 74\%, which is close to the maximum achievable accuracy of 75\%.

\subsection{OOD generalization under label shift}

\noindent In order to explore the robustness of our proposed algorithm to label shift, we designed an imbalanced Colored MNIST dataset which is the same as the Colored MNIST dataset but we modified the dataset in a way that each of the two domains contain 95\% of data points from one class and 5\% from the other class. In fact, in the imbalanced Colored MNIST dataset, both correlation shift and label shift exist. We repeated the experiments in Table \ref{table:Table2} on imbalanced Colored MNIST and reported the results in Table \ref{table:Table3}. Looking at Table \ref{table:Table3}, in the heavy label shift case, the Hutchinson achieves the highest OOD accuracy which outperforms the best-performing algorithm by a margin of 12\%.

\begin{table}[ht]
\caption{Comparison of ERM, IRM, V-Rex, Fishr, and our proposed methods HGP and Hutchinson on Imbalanced Colored MNIST where each domain has $95\%$ from one class and $5\%$ from other class. Except for the imposed label shift, the setting is same as Colored MNIST experiment introduced in IRM \citep{arjovsky2019invariant}.}
  \label{table:Table3}
  \centering
  \begin{tabular}{lcc}
    \toprule
     Method & Train acc. & Test acc.    \\
    \toprule
        ERM &  86.4 ± 0.1 & 16.7 ± 0.1   \\ 
        IRM & 84.9  ±  0.1 &14.3  ± 1.4   \\ 
        V-REx & 83.3 ± 0.2  & 35.1 ± 1.2   \\ 
        Fishr &  75.7 ± 2.7  &  35.5 ± 5.3   \\ 
        HGP &  83.0  ± 0.2  & 30.0  ± 0.9   \\ 
        Hutchinson &  79.4  ± 0.3  & \textbf{47.7  ± 1.4}   \\ 
    \bottomrule
  \end{tabular}
\end{table}

\subsection{Training dynamics for Colored MNIST and Imbalanced Colored MNIST}

\noindent Similar to earlier works \citep{arjovsky2019invariant,krueger2021out,rame2022fishr}, we study the training and OOD generalization dynamics of HGP and Hutchinson on CMNIST and imbalanced CMNIST datasets. Figure \ref{figure:fig2-a} and Figure \ref{figure:fig3-a} show the Hessian distance $\| \Hv_{\Sc_1}  - \Hv_{\Sc_2} \|_F$ and OOD accuracy/loss of HGP and Hutchinson during the training. From Figure \ref{figure:fig2-a} when regularization parameters increase at step 190, the Hessian distance significantly drops and immediately OOD accuracy increases which suggests a strong negative correlation between the Hessian distance and OOD accuracy. This correlation is more evident for Hutchinson compared to HGP since after step 300, the Hessian distance keeps decreasing and OOD accuracy increases correspondingly. In contrast, for HGP, after step 300 when we are close to convergence, OOD accuracy becomes stable and even there is a slight increase in the Hessian distance. This is because close to convergence the gradients become zero and HGP loss cannot align Hessians as well as Hutchinson.

\begin{figure}[h]
\centering
\centering\includegraphics[width=.9\linewidth]{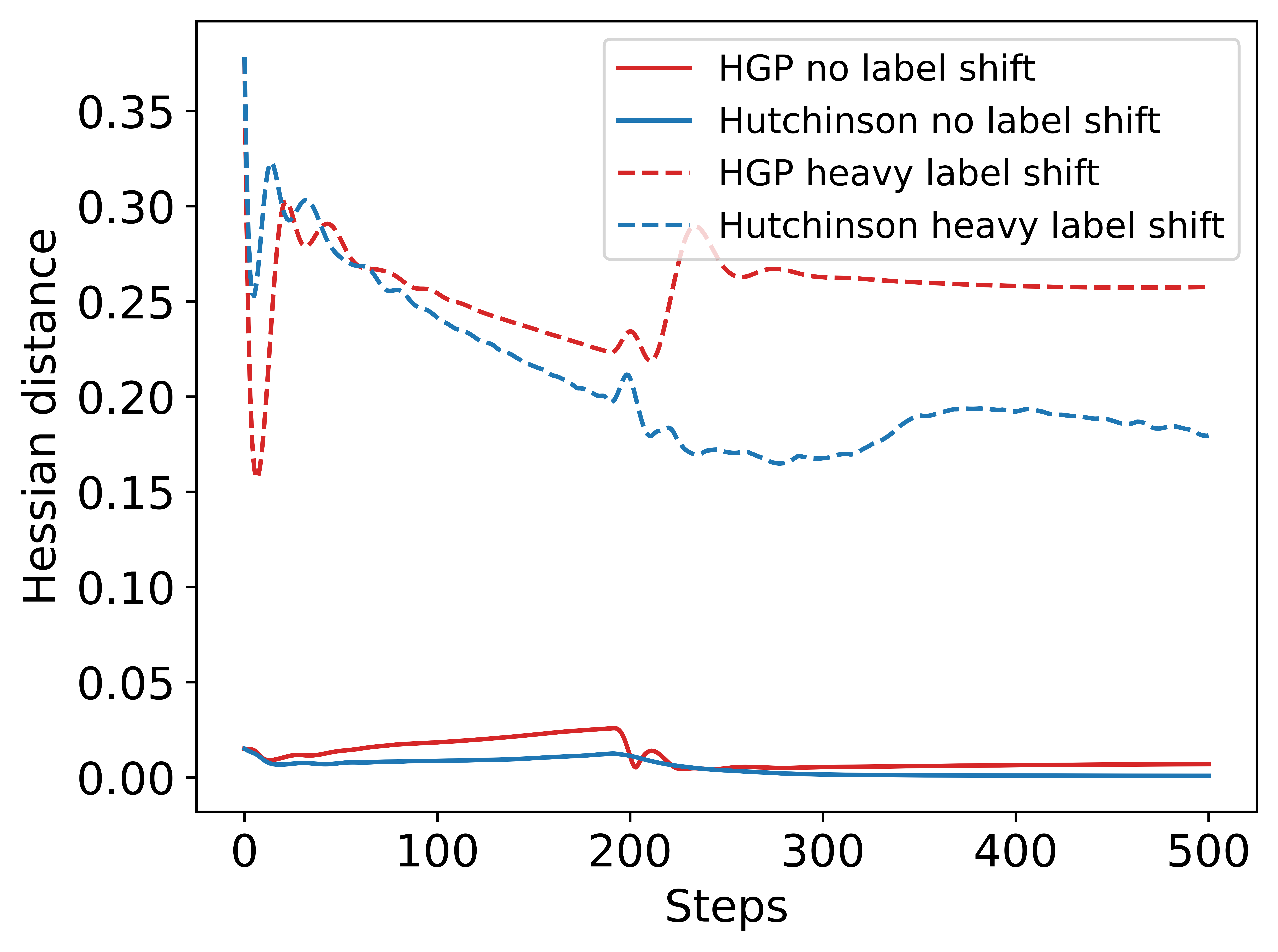}
\caption{The Hessian distances $\| \Hv_{\Sc_1}  - \Hv_{\Sc_2} \|_F$ in Hutchinson and HGP regularization losses during the training steps for no label shift and heavy label shift in Colored MNIST experiment. Clearly, for heavy label shift, the Hessian distance is significantly more than no label shift case.}
\label{fig:4}
\end{figure}

For the heavy label shift case, Figure \ref{figure:fig2-b} and \ref{figure:fig3-b} show that it is much more difficult to decrease the Hessian distance in general. Besides, Hutchinson is more capable of reducing the Hessian distance compared to HGP and subsequently, it achieves better OOD accuracy.

Finally, Figure~\ref{fig:4} compares Hessian distance while we use HGP and Hutchinson regularizers in no label shift and heavy label shift datasets. Clearly, for heavy label shift, the Hessian distance is significantly larger which shows that transferability between domains becomes more difficult. This experiment validates our argument that Hessian distance captures transferability between domains.

\subsection{OOD generalization under diversity shift}

\noindent In this section we use the Domainbed benchmark \citep{gulrajani2020search} to evaluate OOD accuracy of HGP and Hutchinson. We validate our proposed algorithms on the VLCS \citep{fang2013unbiased}, PACS \citep{li2017deeper},  OfficeHome \citep{venkateswara2017deep}, and DomainNet \citep{peng2019moment} datasets from the DomainBed benchmark. Domainbed provides a standard scheme to validate the performance of DG algorithms in a fair manner with different model selections. Table \ref{table:leave_one_out} provides the performance of HGP and Hutchinson methods against multiple benchmarks for leave-one-domain-out model selection reported from Domainbed. It shows that our Hutchinson method gives the state-of-the-art OOD performance in most cases. For HGP, our results are not as good as expected since the gradients diminish near convergence. The most competitive method to Hutchinson is CORAL, which we showed approximately aligns Hessians (by matching covariances) and gradients (by matching features). The results for other model selections are reported in the supplementary material.

\begin{table*}[h]
\caption{DomainBed benchmark with leave-one-domain-out cross-validation model selection for CMNIST, VLCS, PACS, and OfficeHome datasets.  We show the best and second best number with boldface and underline respectively. Due to time and computation limit, we did not finish the Fishr experiment on DomainNet. }
\label{table:leave_one_out}
\centering
\begin{tabular}{lccccc}
\toprule
\textbf{Algorithm}                & \textbf{VLCS}             & \textbf{PACS}             & \textbf{OfficeHome}   & \textbf{DomainNet}         & \textbf{Avg}              \\
\midrule
ERM   \citep{vapnik1999overview}                                      & 77.2             & 83.0            & 65.7       & 40.6               & 66.6                     \\
IRM \citep{arjovsky2019invariant}                                      & 76.3            & 81.5            & 64.3      &33.5          & 63.9                     \\
GroupDRO  \citep{sagawa2019distributionally}                                  & 77.9          & \underline{83.5}            & 65.2      &33.0      & 64.9                     \\
Mixup   \citep{wang2020heterogeneous}                                & 77.7            & 83.2             & 67.0           &38.5       & 66.6                    \\
MLDG   \citep{li2018learning}                               & 77.2            & 82.9            & 66.1      &41.0          & 66.8                    \\
CORAL \citep{sun2016deep}                                      & \underline{78.7}         & 82.6             & \textbf{68.5}     &\underline{41.1}      &\underline{67.7}                   \\
MMD    \citep{li2018domain}                                       & 77.3           & 83.2            & 60.2              &23.4         & 61.0                   \\
DANN  \citep{ganin2016domain}                                      & 76.9            & 81.0            & 64.9        & 38.2         & 65.2                    \\
CDANN \citep{zhou2021domain}                                       & 77.5           & 78.8            & 64.3         &38.0            &  64.6        \\
MTL    \citep{blanchard2021domain}                                & 76.6               & 83.7             & 65.7     & 40.6    &66.7  \\
SagNet   \citep{nam2020learning}                                 & 77.5            & 82.3             & 67.6            & 40.2             & 66.9                     \\
ARM   \citep{zhang2020adaptive}                                   & 76.6           & 81.7             & 64.4               & 35.2       & 64.5                  \\
VREx    \citep{krueger2021out}                           & 76.7           & 81.3             & 64.9              & 33.4     & 64.1               \\
RSC   \citep{huang2020self}                             & 77.5          & 82.6             & 65.8        & 38.9        & 66.2                 \\

Fishr   \citep{rame2022fishr}                              & 78.2          & \textbf{85.4}       & \underline{67.8}   &  - & -  \\
\midrule
 HGP                                          & 76.7           & 82.2            & 67.5         &\underline{41.1}      &66.9                    \\
 Hutchinson                              & \textbf{79.3}           & \underline{84.8}            & \textbf{68.5}           &\textbf{41.4}     & \textbf{68.5}                  \\
 
\bottomrule
\end{tabular}
\end{table*}

\begin{table}[!h]
\caption{Training time comparison on CMNIST between ERM, IRM, V-Rex, Fishr, and our proposed methods HGP and Hutchinson for 100 steps averaged over 10 runs.}
  \label{table:Table5}
  \centering
  \begin{tabular}{lc}
    \toprule
     Method & Train acc.     \\
    \toprule
       ERM & 15.3 ± 0.1   \\ 
        CORAL & 32.0 ± 3.2  \\
        V-REx & 34.1 ± 2.3  \\ 
        IRM & 27.0 ± 4.6   \\ 
        GroupDRO & 32.8 ± 2.9  \\ 
        Fishr & 25.3 ± 2.1  \\ 
        \midrule
        HGP & 39.5 ± 4.4   \\
        Hutchinson & 62.4 ± 1.5   
       \\
    \bottomrule
  \end{tabular}
\end{table}

\subsection{Robustness to adversarial shift}

Robustness of deep learning models to adversarial attacks is a notion of their generalization ability. To explore adversarial robustness of models trained using HGP and Hutchinson loss, we evaluate their performance against the Fast Gradient Sign Method (FSGM) attack \citep{goodfellow2014explaining} and benchmark their performances against, ERM, IRM, V-Rex, and Fishr on CMNIST in \Cref{fig:5}. We see that our Hutchinson method is also more transferable to adversarially perturbed domains. 

\begin{figure}[h]
\centering
\centering\includegraphics[width=.9\linewidth]{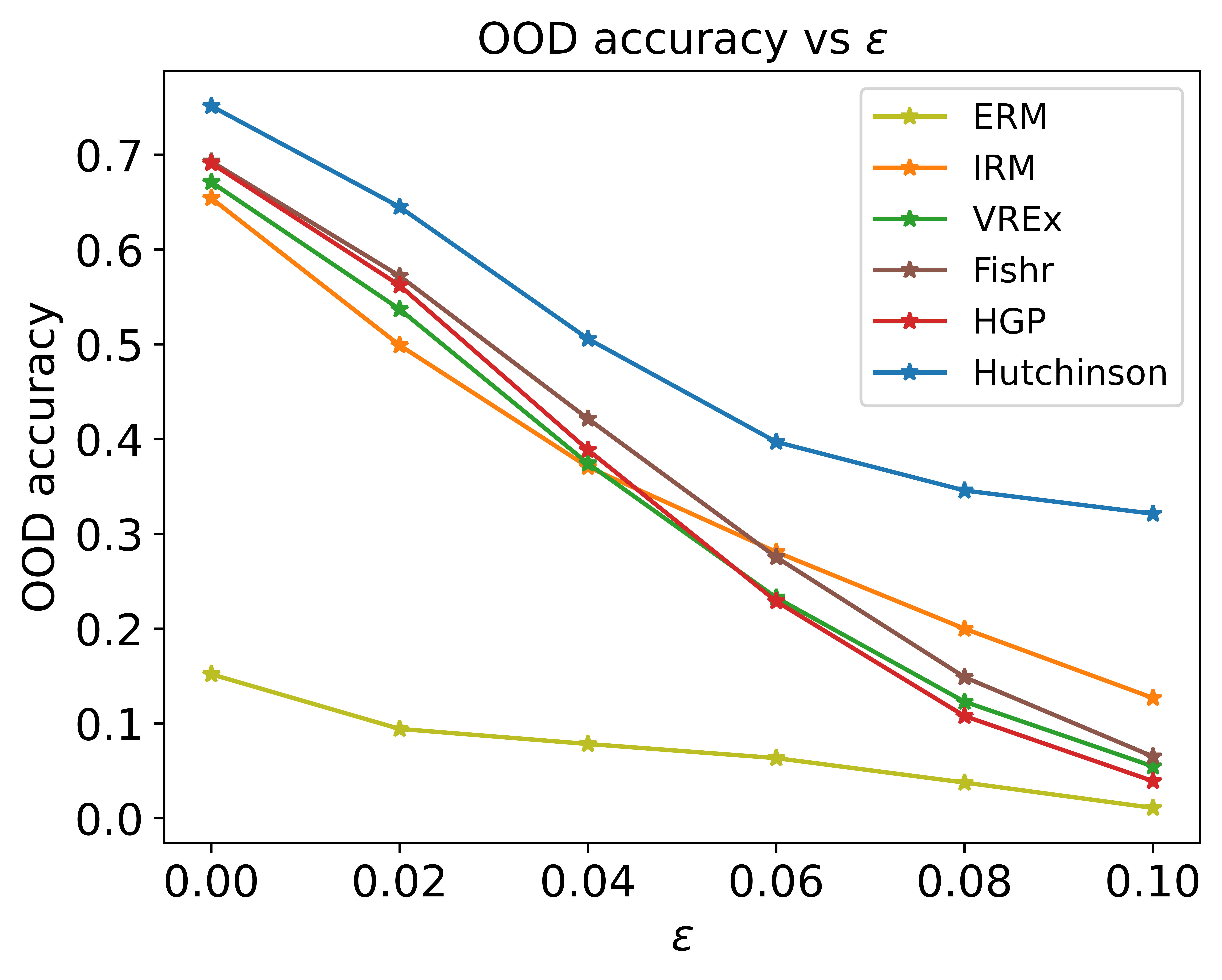}
\caption{Comparison of OOD accuracy under adversarial attack. $\varepsilon$ denotes the perturbation amplitude.}
\label{fig:5}
\end{figure}

\subsection{Computational time comparison}

For CMNIST experiments, we recorded the training time of our proposed algorithms (HGP and Hutchinson) and several popular domain generalization algorithms. The results of this experiment are presented in Table~\ref{table:Table5}. We note that our Hessian alignment methods have more computation cost than existing algorithms, due to the Hessian estimation. This is the expense of better OOD generalization. However, the additional cost is controllable.

\begin{table}[!h]
\caption{Ablation study of the Hutchinson method on different alignment regularization on the PACS dataset when the test domain is sketch.}
  \label{table:Table6}
  \centering
  \begin{tabular}{lc}
    \toprule
     Method & Test acc.     \\
    \toprule
        Hessian \& gradient &  81.4   \\ 
        Gradient only  & 77.0\\
        Hessian only   & 79.4   \\
    \bottomrule
  \end{tabular}
\end{table}
 
\subsection{The roles of Hessian and gradient alignment}

\noindent Since we have implemented both Hessian and gradient alignment in our algorithms, which part is playing a main role? In this section, we conduct an ablation study to investigate the role of each regularization term in the loss functions of Hessian alignment. Recall that $\alpha$ is the regularization weight of Hessian alignment and $\beta$ is the weight of the gradient alignment.
We set up an experiment with the PACS dataset and set the Sketch domain as test domain. We compare the Hutchinson algorithm in three different scenarios:
\begin{itemize}
\item Both: $\alpha=10^{-5}, \beta=10^{-5}$;
\item Only Hessian alignment: $\alpha=10^{-5}, \beta=0$;
\item  Only gradient alignment: $\alpha=0, \beta=10^{-5}$.
\end{itemize}
 The test domain accuracies for models trained with different hyperparameters are presented in Table~\ref{table:Table6}. Clearly, with both Hessian and gradient alignment we achieve the best OOD generalization. Moreover, Hessian alignment plays a more important role than gradient alignment. Additional ablation studies for CMNIST experiment are presented in the appendix.

\subsection{Implementation details}

\noindent We provide the implementation details for all the experiments we conducted in the paper. For the transferability experiments in \Cref{fig:attack}, the experimental setup is exactly the same as the one used in \citet{zhang2021quantifying}. For regularization parameters, we used $\alpha = 10^{-5}$ and $\beta = 10^{-3}$ for HGP and Hutchinson methods. For the CMNIST experiment, the experimental setting is exactly the same as the IRM paper \citep{arjovsky2019invariant}. To set the regularization parameters, we also used the same scheme employed in the IRM paper \citep{arjovsky2019invariant}. The only difference here is that we have two regularization parameters so we set $\alpha = \beta$. For DomainBed experiments, the setup is the same as in \citet{gulrajani2020search}. For $\alpha$ and $\beta$, we used ranges $(-3, -5)$, $(-1, -3)$ for HGP and $(-1, -3)$, $(-2, -4)$ for Hutchinson, respectively.

\section{Conclusions}

\noindent In this paper, we study the role Hessian alignment in domain generalization. Hessian alignment of the classifier head does not only improve the transferability to new domains, but also serves as an effective feature matching mechanism. In order to overcome the difficulty of heavy computation, we propose two estimation methods for Hessians, using Hessian-gradient product and Hutchinson's diagonal estimation. These methods are tested to be effective in various scenarios, including spurious correlation (CMNIST), standard diversity shift benchmark (Domainbed), and adversarial shift. At the expense of slightly heavier computation, Hessian alignment performs competitively, often achieving the state-of-the-art. To our knowledge, our method provides the first Hessian alignment in domain generalization. In the future, it would be interesting to study and compare more efficient ways to align Hessians and gradients across domains.

\section*{Acknowledgements}

We want to thank the anonymous ICCV area chair and reviewers for constructive feedback. GZ would like to thank Han Zhao for valuable comments of the camera ready draft.

\clearpage

{\small
\bibliographystyle{apalike}
\bibliography{egbib}
}

\appendix

\clearpage

\onecolumn

\section{Proof of Proposition \ref{prop:1}}

\begin{proof}

To formulate the classifier head of the neural network, let $z_i$ be the $i$-th component of the feature vector before the classifier layer and the classifier's parameter $\thetav$ is decomposed to $w_{k,i}$, the element in row $k$ and column $i$ of the classifier weight matrix, and $b_k$, the bias term for the $k$-th output. We define $a_k$ as 
\begin{equation} 
a_k= \sum_i^c w_{k,i} z_i + b_k,
\label{eq:21} 
\end{equation}
where $c$ is the number of classes. Given $a_k$, if we assume the classifier activation function $\sigma(\cdot)$ to be softmax, the classifier output for the $k$-th neuron can be written as
\begin{equation} 
\hat{y}_k = \sigma(a_k) = \frac{e^{a_k}}{\sum_{j=1}^c e^{a_j}},
\label{eq:22} 
\end{equation}
Now, if we denote $(\xv, \yv)$ as the sample and $\hat{\yv}$ as the associated output, and assume the loss function be cross-entropy
\begin{equation} 
\ell(\xv, \yv;\thetav) = - \sum_{l=1}^c y_l \log \hat{y}_l,
\label{eq:23} 
\end{equation}
Having the loss function, we proceed to calculate the gradients and hessian of loss with respect to classifier parameters $w_{p,q}$ and $b_p$ is 
\begin{align} 
&\frac{\partial \ell(\hat{\yv}, \yv;\thetav)}{\partial w_{p,q}} = \sum_{l=1}^c \frac{\partial \ell}{\partial \hat{y}_l} \sum_{k=1}^c \frac{\partial  \hat{y}_l}{\partial a_k} \frac{\partial a_k}{\partial w_{p,q}}.
\label{eq:24}
\end{align}

\begin{align} 
&\frac{\partial \ell (\hat{\yv}, \yv;\thetav)}{\partial b_u} = \sum_{l=1}^{c}\frac{\partial  \ell}{\partial \hat{y}_l} \sum_{k=1}^c \frac{\partial  \hat{y}_l}{\partial a_k} \frac{\partial a_k}{\partial b_u}.
\label{eq:25} 
\end{align}

Given that the activation function is a softmax function $\hat{y}_l = \sigma(a_l) = \frac{e^{a_l}}{\sum_j e^{a_j}}$, the $\frac{\partial \sigma(a_l)}{\partial a_k}$ can be calculated as:
\begin{equation} 
\frac{\partial \sigma(a_l)}{\partial a_k} = \sigma(a_l) \delta_{l,k} - \sigma(a_l) \sigma(a_k)
\label{eq:26}.
\end{equation}

Now, using Eq. \ref{eq:26}, we can rewrite Eqs. \ref{eq:24} and \ref{eq:25} as follows:
\begin{equation} 
\frac{\partial  \ell}{\partial w_{p,q}} = (\hat{y}_p - y_p) z_q , \, \frac{\partial  \ell}{\partial b_u} = (\hat{y}_u - y_u).
\label{eq:28} 
\end{equation}

For the Hessian, we only calculate the elements of matrix that are only related to the classifier layer. More precisely, we calculate  $\frac{\partial ^2  \ell}{\partial w_{u,v} \partial w_{p,q}}$,  $\frac{\partial ^2  \ell}{\partial w_{p,q} \partial b_u}$, and $\frac{\partial ^2  \ell}{\partial b_u \partial b_v}$. 
For the $\frac{\partial ^2  \ell}{\partial w_{u,v} \partial w_{p,q}}$ we can write

\begin{equation} 
\begin{aligned}
\frac{\partial ^2  \ell}{\partial w_{u,v} \partial w_{p,q}} = \frac{\partial}{\partial w_{u,v}} \left((\hat{y}_p - y_p) z_q \right). 
\label{eq:29} 
\end{aligned}
\end{equation}

To calculate the above expression, we need $\frac{\partial \hat{y}_p}{\partial w_{u,v}}$:
\begin{align}
\frac{\partial \hat{y}_p}{\partial w_{u,v}} = \sum_{k} \frac{\partial  \hat{y}_p}{\partial a_k} \frac{\partial a_k}{\partial w_{u,v}} = \sum_{k} \frac{\partial \sigma(a_p)}{\partial a_k} \sum_{i} \delta_{k,u} \delta_{i,v} z_i = \sum_{k} \sigma(a_p)(\delta_{p,k}-\sigma(a_k)) \delta_{k,u} z_v = \hat{y}_p z_v (\delta_{p,u}-\hat{y}_u)
\label{eq:30} 
\end{align}
Now, incorporating eq.~\ref{eq:30} into eq.~\ref{eq:29}, the elements of Hessian matrix in classifier layer i.e., $\frac{\partial ^2  \ell}{\partial w_{u,v} \partial w_{p,q}}$, we have:
\begin{align}
\frac{\partial ^2  \ell}{\partial w_{u,v} \partial w_{p,q}} =z_q z_v \hat{y}_p (\delta_{p,u}-\hat{y}_u).   
\label{eq:31} 
\end{align}
For $\frac{\partial ^2  \ell}{\partial w_{p,q} \partial b_u}$ we can write
\begin{align}
    \frac{\partial ^2  \ell}{\partial w_{p,q} \partial b_u} = \frac{\partial}{\partial b_u} \left((\hat{y}_p - y_p) z_q \right) =  \frac{\partial (\hat{y}_p - y_p)}{\partial b_u} z_q  = 
    \sum_{k} \frac{\partial \hat{y}_p}{\partial a_k} \frac{\partial a_k}{\partial b_u} z_q = z_q \hat{y}_p (\delta_{p,u}-\hat{y}_u).
\label{eq:32}
\end{align}
Eventually, for $\frac{\partial ^2 \mathcal{L}}{\partial b_u \partial b_v}$ we have:
\begin{align}
    \frac{\partial^2  \ell}{\partial b_u \partial b_v} = \frac{\partial}{\partial b_v} \left(\hat{y}_u - y_u \right) =
    \sum_k \frac{\partial \hat{y}_u}{\partial a_k} \frac{\partial a_k}{\partial b_v}  = 
    \hat{y}_u (\delta_{u,v}-\hat{y}_v). 
\label{eq:33} 
\end{align}
\end{proof}

\section{Alignment attributes in Hessian and gradient in regression task with mean square error loss and general activation function}

In this section, we extend our analysis to regression tasks for real numbers. We show that the classifier head's gradient and Hessian yield similar information of the features. To adapt our framework to regression, we replace the meaning of $\sigma$ from softmax to an arbitrary uni-variate activation function.

\begin{restatable}[\textbf{Alignment attributes in Hessian and gradient for mean square error loss and general activation function}]{prop}{prop2_with_proof}\label{prop:MSE} 

Let $\hat{y}$ and $y$ be the network prediction and true target associated with the output neuron of a single output network, $\sigma(\cdot)$ be the activation function, $z_i$ be the $i$-th feature value before the last layer (regression layer). Suppose the last layer's parameter $\thetav$ is decomposed to $w_i$, the $i$-th element of the weight vector, and $b$, the bias term. Matching the gradients and Hessians with respect to the last layer across the domain aligns the following attributes 
\begin{align}
\frac{\partial \ell}{\partial b} &= (\hat{y} - y) \sigma'(a),  \label{eq:34} \\
\frac{\partial  \ell}{\partial w_i} &=  (\hat{y} - y) \sigma'(a) z_i,  \label{eq:35} \\
\frac{\partial ^2 \ell}{\partial b^2} &=\sigma'(a)^2  + (\hat{y} - y) \sigma''(a),  \label{eq:36}\\
\frac{\partial ^2 \ell}{\partial w_i \partial b} &= \sigma'(a)^2 z_i + (\hat{y} - y) \sigma''(a) z_i,  \label{eq:37}\\
\frac{\partial ^2 \ell}{\partial w_i \partial w_k} &= \sigma'(a)^2 z_i  z_k + (\hat{y} - y) \sigma''(a) z_i z_k,  \label{eq:38}
\end{align}
   
\end{restatable}

\begin{proof}

To formulate the last layer of the neural network we define $a$ as

\begin{equation} 
a= \sum_i w_i z_i + b.
\label{eq:39} 
\end{equation}

\noindent Given $a$, if we assume the last layer activation function is $\sigma(\cdot)$, the single output can be written as

\begin{equation} 
\hat{y} = \sigma(a).
\label{eq:40} 
\end{equation}

\noindent Now, if we denote the input data as $(\xv, y)$  and associated output $\hat{y}$, and assume the loss function be
\begin{equation} 
\ell(\xv, y;\thetav) = \frac{1}{2} (\hat{y} - y)^2.
\label{eq:41} 
\end{equation}

Having the loss function, we proceed to calculate the gradients and hessian of loss with respect to last layer parameters $w_i$ and $b$. For the gradients, we write
\begin{align} 
\label{eq:42} \frac{\partial \ell} {\partial w_i} &= \frac{\partial \mathcal{L}}{\partial \hat{y}} \frac{\partial  \hat{y}}{\partial a} \frac{\partial a}{\partial w_i} = (\hat{y} - y) \sigma'(a) z_i, \\
\frac{\partial \ell} {\partial b} &= \frac{\partial \mathcal{L}}{\partial \hat{y}} \frac{\partial  \hat{y}}{\partial a} \frac{\partial a}{\partial b} = (\hat{y} - y) \sigma'(a).
\label{eq:43} 
\end{align}

For the Hessian matrix, we only calculate the elements of the matrix that are only related to the last layer. More precisely, we calculate  $\frac{\partial ^2 \mathcal{L}}{\partial w_i \partial w_k}$,  $\frac{\partial ^2 \mathcal{L}}{\partial w_i \partial b}$, and $\frac{\partial ^2 \mathcal{L}}{\partial b^2}$. 
For the $\frac{\partial ^2 \mathcal{L}}{\partial w_i \partial w_k}$ we can write

\begin{equation} 
\begin{aligned}
\frac{\partial ^2 \ell} {\partial w_i \partial w_k} = \frac{\partial}{\partial w_k} \left((\hat{y} - y) \cdot \sigma'(a) z_i \right) = \frac{\partial (\hat{y} - y)}{\partial w_k} \cdot \sigma'(a) z_i + (\hat{y} - y) \frac{\partial}{\partial w_k} \sigma'(a)  z_i
\label{eq:44} 
\end{aligned}
\end{equation}

To calculate the above expression, we need $\frac{\partial \hat{y}}{\partial w_k}$ and $\frac{\partial}{\partial w_k} \sigma'(a)$.

\begin{align}
&\frac{\partial \hat{y}}{\partial w_k} =  \frac{\partial  \hat{y}}{\partial a} \frac{\partial a}{\partial w_k} = \sigma'(a) z_k, \\
 & \frac{\partial}{\partial w_k} \sigma'(a) = \frac{\partial}{\partial a} \sigma'(a) \frac{\partial a}{\partial w_k}= \sigma''(a) z_k. 
\label{eq:46} 
\end{align}

Now, incorporating eq.~\ref{eq:46} into \ref{eq:44} the elements of the last-layer Hessian matrix become:
\begin{align}
\frac{\partial ^2 \ell}{\partial w_i \partial w_k} = \sigma'(a)^2 z_i  z_k + (\hat{y} - y) \sigma''(a) z_i z_k. 
\label{eq:47} 
\end{align}

For $\frac{\partial ^2 \mathcal{L}}{\partial w_i \partial b}$ we can write
\begin{align}
    \frac{\partial ^2 \ell}{\partial w_i \partial b} &= \frac{\partial}{\partial b} \left((\hat{y} - y) \cdot \sigma'(a) z_i \right) =   \frac{\partial (\hat{y} - y)}{\partial b} \cdot \sigma'(a) z_i  + (\hat{y} - y) \frac{\partial}{\partial b} \sigma'(a) z_i  \\
    & =
    \frac{\partial  \hat{y}}{\partial a} \frac{\partial a}{\partial b} \cdot \sigma'(a) z_i + (\hat{y} - y)\frac{\partial}{\partial a} \sigma'(a) \frac{\partial a}{\partial b} z_i = 
    \sigma'(a)^2 z_i + (\hat{y} - y) \sigma''(a) z_i.
\label{eq:48} 
\end{align}

Eventually, $\frac{\partial ^2 \mathcal{L}}{\partial b^2}$ is calculated as:

\begin{align}
    \frac{\partial ^2 \ell}{\partial b^2} &= \frac{\partial}{\partial b} \left(\hat{y} - y) \cdot \sigma'(a) \right) =  \frac{\partial (\hat{y} - y)}{\partial b} \cdot \sigma'(a)  + (\hat{y} - y) \frac{\partial}{\partial b} \sigma'(a)  \\
     &= \frac{\partial  \hat{y}}{\partial a} \frac{\partial a}{\partial b}\cdot \sigma'(a) + (\hat{y} - y)\frac{\partial}{\partial a} \sigma'(a) \frac{\partial a}{\partial b} = 
    \sigma'(a)^2  + (\hat{y} - y) \sigma''(a). 
\label{eq:49} 
\end{align}

\end{proof}

Eqs.~\ref{eq:43}, \ref{eq:47}, \ref{eq:48}, \ref{eq:49} show that matching Hessians and gradients with respect to the last layer parameters will match neural network outputs, last layer input features and covariance between output features across domains. This supports the idea of matching gradients and Hessians during training. The above result can be extended to multi-dimensional outputs if the activation is element-wise.

\section{Ablation Study}

\begin{table}[h]
\caption{Comparison of ERM, IRM, V-Rex, Fishr, and our proposed methods HGP and Hutchinson with ablation study on $\alpha$ and $\beta$ on Colored MNIST. The setting is same as the Colored MNIST experiment introduced in IRM \citep{arjovsky2019invariant}.}
  \label{table:Table9}
  \centering
  \begin{tabular}{lcc}
    \toprule
     Method & Train acc. & Test acc.    \\
    \toprule
        ERM & 86.4 ± 0.2  & 14.0 ± 0.7   \\ 
        IRM & 71.0 ± 0.5   &  65.6 ± 1.8   \\ 
        V-REx & 71.7 ± 1.5  & 67.2 ± 1.5  \\ 
        Fishr &  71.0 ± 0.9  &  69.5 ± 1.0   \\ 
        HGP &  71.0 ± 1.5   & 69.4 ± 1.3   \\ 
        HGP ($\alpha=0$) & 70.6 ± 1.8  &69.3 ± 1.2  \\
        HGP ($\beta=0$) & 78.9 ± 0.3   &53.3 ± 1.7   \\
        Hutchinson &  61.7 ± 1.9    & \textbf{74.0 ± 1.2}   \\ 
         Hutchinson ($\alpha=0$)& 70.6 ± 1.8 & 69.3 ± 1.2 \\
        Hutchinson  ($\beta=0$) & 84.9 ± 0.1 & 9.8 ± 0.2  \\
    \bottomrule
  \end{tabular}
\end{table}

\begin{table}[h]
\caption{Comparison of ERM, IRM, V-Rex, Fishr, and our proposed methods HGP and Hutchinson with ablation study on $\alpha$ and $\beta$ on Imbalanced Colored MNIST where each domain has $95\%$ from one class and $5\%$ from other class. Except for the imposed label shift, the setting is same as the Colored MNIST experiment introduced in IRM \citep{arjovsky2019invariant}.}
  \label{table:Table10}
  \centering
  \begin{tabular}{lcc}
    \toprule
     Method & Train acc. & Test acc.    \\
    \toprule
        ERM &  86.4 ± 0.1 & 16.7 ± 0.1   \\ 
        IRM & 84.9  ±  0.1 &14.3  ± 1.4   \\ 
        V-REx & 83.3 ± 0.2  & 35.1 ± 1.2   \\ 
        Fishr &  75.7 ± 2.7  &  35.5 ± 5.3   \\ 
        HGP &  83.0  ± 0.2  & 30.0  ± 0.9   \\ 
        HGP ($\alpha=0$) & 83.2 ± 0.4 &29.7 ± 1.7  \\
        HGP ($\beta=0$) & 84.3 ± 0.1  &20.4  ± 1.0   \\
        Hutchinson &  79.4  ± 0.3  & \textbf{47.7  ± 1.4}   \\ 
         Hutchinson ($\alpha=0$)& 83.2 ± 0.4 & 29.7 ± 1.7 \\
        Hutchinson  ($\beta=0$) & 84.9 ± 0.1 & 9.8 ± 0.2  \\
    \bottomrule
  \end{tabular}
\end{table}

We also repeat the Colored MNIST and imbalanced Colored MNIST experiments in scenarios where one of $\alpha$ and $\beta$ is non-zero, in Table~\ref{table:Table9} and Table~\ref{table:Table10}. Recall that $\alpha$ controls the Hessian alignment and $\beta$ controls the gradient alignment. If not mentioned, the values for $\alpha$ and/or $\beta$ are non-zero and they are chosen exactly as the IRM paper \citep{arjovsky2019invariant}. According to Table~\ref{table:Table9}, for both HGP and Hutchinson methods, gradient alignment seems to contribute more to OOD generalization on CMNIST. This might be due to the heavy correlation shift and we have to align the local minima first. For Hutchinson, when $\beta=0$, the OOD performance drops which we believe is because the value that has been chosen for $\alpha$ is optimized for the IRM loss. In other words, if we optimize $\alpha$ for aligning the diagonal part of Hessian, it can contribute to the OOD generalization. For imbalanced Colored MNIST, the same trend for the role of $\alpha$ and $\beta$ can be observed. Overall, the key observation is that both aligning gradients and diagonal parts of Hessians contribute to the OOD generalization.

\section{Domainbed Results for Other Model Section Methods}

\noindent In this section, we provide the Domainbed results for the two other model selection methods, i.e., the training-domain validation set and test-domain validation set (oracle). First note that the oracle model selection is not a valid benchmarking scheme and not applicable in practice as it uses the target domain data for selecting the hyperparameters. In fact, in this scenario, algorithms with more hyperparameters and training tricks (like warmup, exponential moving average and etc) can obtain better performance since they have more freedom to tune the model on test data. Considering this, we should not rely on the oracle model selection technique to compare domain generalization algorithms. 
The other model selection technique is the training-domain validation set where the validation sets of all training domain are concatenated together and select the hyperparameters that maximize the accuracy on the entire validation set.

As can be seen in Table~\ref{table:Table7} and Table~\ref{table:Table8}, although Hessian alignment methods are not the best, their performance across all datasets is still competitive for other model selections. As also shown in \citet{gulrajani2020search}, different model selections could result in different rankings of the algorithms. We find that training-domain model selection in general gives better results for most baseline algorithms, but the performance of the Hutchinson method slightly degrades. We defer the study of comparing model selections to future work.

\begin{table*}[h]
\caption{DomainBed benchmark with \emph{training-domain validation set model selection} method for CMNIST, VLCS, PACS, and OfficeHome datasets. We show the best and second best number with boldface and underline respectively.}
\label{table:Table7}
\centering
\begin{tabular}{lccccc}
\toprule
\textbf{Algorithm}        & \textbf{VLCS}             & \textbf{PACS}             & \textbf{OfficeHome}   & \textbf{DomainNet}             & \textbf{Avg}              \\
\midrule
ERM                      & 77.5            & 85.5         & 66.5   &40.9                      & 67.6                     \\
IRM                       & 78.5            & 83.5              & 64.3      & 33.9             & 65.1                      \\
GroupDRO                  & 76.7          & 84.4            & 66.0              &33.3      & 65.1                  \\
Mixup                      & 77.4            & 84.6             & 68.1             &39.2        & 67.3                   \\
MLDG                      & 77.2             & 84.9            & 66.8          &41.2            & 67.5                     \\
CORAL                   & \textbf{78.8}           & \underline{86.2}             & \textbf{68.7}   &41.5                 &\textbf{68.8}                      \\
MMD                       & 77.5             & 84.6             & 66.3          &23.4            & 63.0                    \\
DANN                  & \underline{78.6}             & 83.6            & 65.9           &38.3        & 66.6                  \\
CDANN                      & 77.5           & 82.6             & 65.8          &38.3          & 66.1                      \\
MTL                       & 77.2           & 84.6             & 66.4        &40.6           & 67.2                      \\
SagNet                        & 77.8             & \textbf{86.3}          & 68.1           &40.3          & 68.1                      \\
ARM                           & 77.6          & 85.1             & 64.8         &35.5              & 65.8                      \\
VREx                    & 78.3           & 84.9             & 66.4            &33.6         & 65.8                    \\
RSC                        & 77.1            & 85.2             & 65.5      &38.9                  & 66.7                  \\
 AND-mask                      & 78.1            & 84.4              & 65.5   &37.2                     &  66.3                  \\

 SAND-mask                       & 77.4            & 84.6             & 65.8                &32.1        & 65.0                  \\

 Fish                        & 77.8            & 85.5             &\underline{68.6}         &\textbf{42.7}               &  \underline{68.7}                  \\
 Fishr                          & 77.8           &85.8            & 67.8             &\underline{41.7}          &  68.3                \\
 \midrule
 HGP                                 & 77.6           & 84.7            & 68.2            &41.1   & 67.9                     \\
 Hutchinson                          & 76.8           & 83.9            & 68.2   &41.6           &   67.6                \\
 
\bottomrule
\end{tabular}
\end{table*}

\begin{table*}[h]
\caption{DomainBed benchmark with \emph{test-domain validation set (oracle)} model selection method for CMNIST, VLCS, PACS, and OfficeHome datasets. We show the best and second best number with boldface and underline respectively.}
\label{table:Table8}
\centering
\begin{tabular}{lcccccc}
\toprule
\textbf{Algorithm}              & \textbf{VLCS}             & \textbf{PACS}             & \textbf{OfficeHome}    & \textbf{DomainNet}             & \textbf{Avg}               \\
\midrule
ERM                              & 77.6             & 86.7           & 66.4    &41.3                  & 68.0                  \\
IRM                                 & 76.9             & 84.5           & 63.0    &28.0                 &  63.1                    \\
GroupDRO                           & 77.4             & \underline{87.1}           & 66.2        &33.4             & 66.0                      \\
Mixup                            & 78.1          & 86.8            & 68.0  &39.6                   &    68.1                   \\
MLDG                          & 77.5             & 86.8           & 66.6      &41.6               & 68.1                      \\
CORAL                              & 77.7             & \underline{87.1}            & \textbf{68.4}              &41.8      & \underline{68.8}                      \\
MMD                                  & 77.9            & \textbf{87.2}             & 66.2      &23.5             & 63.7                    \\
DANN                               & \underline{79.7}              & 85.2          & 65.3  &38.3                   &     67.1                \\
CDANN                                & \textbf{79.9}             & 85.8           & 65.3     &38.5                &  67.4                  \\
MTL                                     & 77.7            & 86.7           & 66.5  &40.8                    & 67.9                      \\
SagNet                           & 77.6           & 86.4             & 67.5      &40.8              &    68.1                  \\
ARM                                & 77.8            & 85.8            & 64.8 &36.0                   &  66.1                   \\
VREx                                 & 78.1            & \textbf{87.2 }         & 65.7 &30.1                      & 65.3                     \\
RSC                           & 77.8             & 86.2            & 66.5       &38.9             &   67.4                  \\

 AND-mask                      & 76.4            & 86.4            & 66.1                        &37.9   &  66.7                 \\

 SAND-mask                     & 76.2            & 85.9             & 65.9     &32.2                  &   65.1                \\

 Fish                       & 77.8           &85.8             & 66.0     & \underline{42.7}                  &    68.1              \\
 Fishr                       & 78.2           &86.9             & \underline{68.2}    &\textbf{43.4}                     &    \textbf{69.2}           \\
 \midrule
 HGP                             & 77.3           & 86.5            & 67.4     & 41.2           & 68.1
               \\
 Hutchinson                     & 77.9          & 86.3            & \textbf{68.4}    &41.9            & 68.6                \\
\bottomrule
\end{tabular}
\end{table*}

\end{document}